\documentclass{article}
\usepackage{graphicx,amssymb,amsmath,amsthm,color,stmaryrd,url,algorithmic, algorithm}
\usepackage{geometry}
\newcommand{\rhog}{{$\rho$G}}

\newtheorem{definition}{Definition}
\newtheorem{proposition}{Proposition}

\begin{document}

\title{RHOG (\rhog): A Refinement-Operator Library\\ for Directed Labeled Graphs\thanks{This research was partially supported by NSF award IIS-1551338.}}


\author{Santiago Onta\~{n}\'{o}n \\ 
		Department of Computer Science, \\ Drexel University, Philadelphia, USA \\ \textsf{\em santi@cs.drexel.edu}}


%\institute{Department of Computer Science, \\ Drexel University, Philadelphia, USA \\ \textsf{\em santi@cs.drexel.edu}}


\maketitle

\begin{abstract}
This document provides the foundations behind the functionality provided by the \rhog\ library\footnote{\url{https://github.com/santiontanon/RHOG}}, focusing on the basic operations the library provides: subsumption, refinement of directed labeled graphs, and distance/similarity assessment between directed labeled graphs. \rhog\ development was initially supported by the National Science Foundation, by the EAGER grant IIS-1551338.

%\keywords{Similarity Measures \and Refinement operators \and Feature Terms \and Structured Machine Learning}
\end{abstract}

\newpage

\tableofcontents

%------------------------------------------------------------------------
%------------------------------------------------------------------------
%------------------------------------------------------------------------

\newpage
\section{Preliminaries and Definitions}

%\subsection{Representation Formalism}
%In the rest of this document, we will concern with distance and similarity functions between two {\em instances} $g_1$, $g_2$ of a given structured machine learning dataset. We will represent instances as {\em directed labeled graphs} (DLGs). 

\subsection{Notation Convention}

In the remainder of this document, we have used the following notation convention:

\begin{itemize}
\item We use capital letters to represent sets, e.g., $V$, and lower case letters to represent the elements of those sets, e.g.: $V = \{v_1, v_2, v_3\}$.
\item We use curly braces to represent sets, e.g.: $\{v_1, v_2, v_3\}$.
\item We use square brackets to represent ordered sequences, e.g.: $[v_1, v_2, v_3, v_1, v_2]$ (notice that, unlike a set, an ordered sequence might contain an element more than once).
\item We use the power notation $2^V$ to represent the set of all possible subsets of a given set $V$.
\item We use the regular expression notation $V^*$ to represent the set of all possible sequences made out of elements of a set $V$.
\end{itemize}

\subsection{Directed Labeled Graphs}

\begin{definition}[Directed Labeled Graph]
Given a finite set of labels $L$, a directed labeled graph $g$ is defined as a tuple $g = \langle V, E, l \rangle$, where:
\begin{itemize}
\item $V = \{v_1, ..., v_n\}$ is a finite set of vertices.
\item $E = \{(v_{i_1}, v_{j_1}), ..., (v_{i_m}, v_{j_m})\}$ is a finite set of edges.
\item $l : V \cup E \to L$, is a function that assigns a label from $L$ to each vertex and edge.
\end{itemize}
\end{definition}

\begin{definition}[Directed Path]
Given a DLG $g = \langle V, E, l \rangle$, we say that there is a directed path from a vertex $v_1 \in V$ to another vertex $v_2 \in V$ when there is a sequence of vertices $[w_1, ..., w_k]$, such that $w_1 = v_1$, $w_k = v_2$, and $\forall 1 \leq j < k : (w_j, w_{j+1}) \in E$.
\end{definition}

\begin{definition}[Undirected Path]
Given a DLG $g = \langle V, E, l \rangle$, we say that there is an undirected path from a vertex $v_1 \in V$ to another vertex $v_2 \in V$ when there is a sequence of vertices $[w_1, ..., w_k]$, such that $w_1 = v_1$, $w_k = v_2$, and $\forall 1 \leq j < k : (w_j, w_{j+1}) \in E \,\, \vee \,\, (w_{j+1}, w_{j}) \in E$.
\end{definition}

Notice that the existence of a directed path implies the existence of an undirected path (since every directed path is also an undirected path). Moreover, when we do not care whether a path is directed or undirected, we will just say ``a path exists'' between two given vertices.

\begin{definition}[Connected DLG]
A directed labeled graph (DLG) $g = \langle V, E, l \rangle$ is {\em connected} when given any two vertices $v_1, v_2 \in V$ there is a path from $v_1$ to $v_2$.
\end{definition}

\begin{definition}[Bridge]\label{def:bridge}
An edge $e \in E$ is a {\em bridge} of a connected DLG $g = \langle V, E, l \rangle$ if the graph resulting from removing $e$ from $g$, $g' = \langle V, E \setminus \{e\}, l \rangle$ is not connected. We will use $\mathit{bridges(g)}$ to denote the set of edges in a graph $g$ that are bridges.
\end{definition}

In the remainder of this document, unless otherwise noted, we will only be interested in connected DLGs. Moreover, we will consider two types of DLGs:
\begin{itemize}
\item {\em Flat-labeled DLGs (FDLG)}: when the set of labels $L$ is a plain set without any relation between the different labels.
\item {\em Order-labeled DLGs (ODLG)}: when the set of labels $L$ is a partially ordered set via a partial-order $\preceq$ such that for any three elements $a, b, c \in L$, we have that:
	\begin{itemize}
	\item $a \preceq a$,
	\item $a \preceq b \,\, \wedge \,\, b \preceq a \implies a = b$,
	\item $a \preceq b \,\, \wedge \,\, b \preceq c \implies a \preceq c$, and
	\item There is a special element $\top \in L$ such that $\forall a \in L : \top \preceq a$.
	\end{itemize}
	When $a \preceq b$ and $b \not\preceq a$, we write $a \prec b$.
\end{itemize}

Intuitively the partial order $\langle L, \preceq \rangle$ can be seen as a multiple-inheritance concept hierarchy with a single top label $\top$ that is more general than all the other labels.

\subsection{Relations Between Graphs}\label{sec:graph-relations}

Let us introduce a collection of relations between directed labeled graphs, which will be used in the remainder of this document. From more to less restrictive:

\begin{definition}[Subgraph/Supergraph]
Given two DLGs, $g_1 = \langle V_1, E_1, l \rangle$ and $g_2 = \langle V_2, E_2, l \rangle$, $g_1$ is said to be a {\em subgraph} of $g_2$ if $V_1 \subseteq V_2$ and $E_1 \subseteq E_2$. We write $g_1 \subseteq g_2$. We will call $g_2$ a {\em supergraph} of $g_1$.
\end{definition}

\begin{definition}[Subsumption]\label{def:subsumption}
Given two DLGs, $g_1 = \langle V_1, E_1, l_1 \rangle$ and $g_2 = \langle V_2, E_2, l_2 \rangle$, $g_1$ is said to {\em subsume} $g_2$ (we write $g_1 \sqsubseteq g_2$) if there is a mapping $m: V_1 \to V_2$ such that:
\begin{itemize}
\item $\forall (v, w) \in E_1: (m(v),m(w)) \in E_2$, 
\item $\forall v \in V_1: l_1(v) = l_2(m(v))$, and
\item $\forall (v, w) \in E_1:  l_1((v,w)) = l_2((m(v),m(w)))$. 
\end{itemize}
\end{definition}

\begin{definition}[Subsumption Relative to $\preceq$]\label{def:subsumption-po}
Given two order-labeled DLGs, $g_1 = \langle V_1, E_1, l_1 \rangle$ and $g_2 = \langle V_2, E_2, l_2 \rangle$, and $\preceq$, the partial order among the labels in $L$, $g_1$ is said to {\em subsume} $g_2$ relative to $\preceq$ (we write $g_1 \sqsubseteq_{\preceq} g_2$) if there is a mapping $m: V_1 \to V_2$ such that:
\begin{itemize}
\item $\forall (v, w) \in E_1: (m(v),m(w)) \in E_2$,
\item $\forall v \in V_1: l_1(v) \preceq l_2(m(v))$, and
\item $\forall (v, w) \in E_1:  l_1((v,w)) \preceq l_2((m(v),m(w)))$. 
\end{itemize}
\end{definition}

Notice that the mapping $m$ between vertices, induces a mapping between edges. So, when $e = (v, w) \in V_1$, we will write $m(e)$ to denote $(m(v), m(w))$.

\begin{definition}[Transitive Subsumption (trans-subsumption)]\label{def:trans-subsumption}
Given two DLGs, $g_1 = \langle V_1, E_1, l_1 \rangle$ and $g_2 = \langle V_2, E_2, l_2 \rangle$, $g_1$ is said to subsume transitively ({\em trans-subsume}) $g_2$ (we write $g_1 \overrightarrow{\sqsubseteq} g_2$) if there is a mapping $m: V_1 \to V_2$, and another mapping $m_e: E_1 \to E_2^*$ ($m_e$ maps edges in $g_1$ to sequences of edges in $g_2$) such that:
\begin{itemize}
\item $\forall v \in V_1: l_1(v) = l_2(m(v))$, and
\item $\forall e = (v_1,v_2) \in E_1: m_e(e) = [e_1 = (w_1, w_2), e_2 = (w_2, w_3), ..., e_k = (w_k, w_{k+1})]$ such that:
	\begin{itemize}
	\item $m(v_1) = w_1$, 
	\item $m(v_2) = w_{k+1}$,
	\item $\forall 1 \leq i \leq k: e_i \in E_2$, and
	\item $\forall 1 \leq i \leq k: l_1(e) = l_2(e_i)$
	\end{itemize}
\end{itemize}
\end{definition}

\begin{definition}[Trans-subsumption Relative to $\preceq$]\label{def:trans-subsumption-po}
Given two order-labeled DLGs, $g_1 = \langle V_1, E_1, l_1 \rangle$ and $g_2 = \langle V_2, E_2, l_2 \rangle$, and $\preceq$, the partial order among the labels in $L$, $g_1$ is said to {\em trans-subsume} $g_2$ relative to $\preceq$ (we write $g_1 \overrightarrow{\sqsubseteq}_{\preceq} g_2$) if there is a mapping $m: V_1 \to V_2$, and another mapping $m_e: E_1 \to V_2^*$ ($m_e$ maps edges in $g_1$ to sequences of vertexes in $g_2$) such that:
\begin{itemize}
\item $\forall v \in V_1: l_1(v) \preceq l_2(m(v))$, and
\item $\forall e = (v_1,v_2) \in E_1: m_e(e) = [e_1 = (w_1, w_2), e_2 = (w_2, w_3), ..., e_k = (w_k, w_{k+1})]$ such that:
	\begin{itemize}
	\item $m(v_1) = w_1$, 
	\item $m(v_2) = w_{k+1}$,
	\item $\forall 1 \leq i \leq k: e_i \in E_2$, and
	\item $\forall 1 \leq i \leq k: l_1(e) \preceq l_2(e_i)$
	\end{itemize}
\end{itemize}
\end{definition}

\begin{figure}[tb]
    \centering
    \includegraphics[width=0.75\columnwidth]{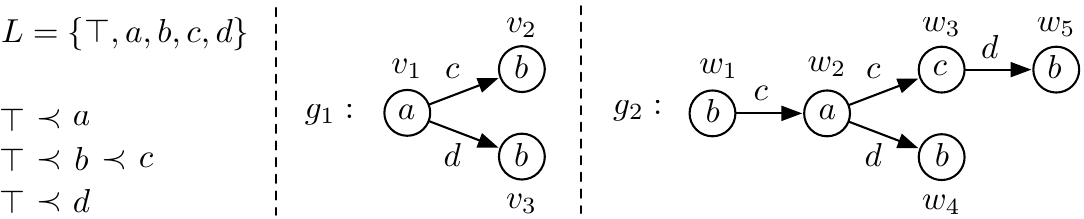}
    \caption{{\bf Subsumption} vs {\bf Subsumption relative to $\preceq$}: Two graphs, $g_1 = \langle V_1 = \{v_1, v_2, v_3\}, E_1 = \{(v_1, v_2), (v_2, v_3)\}, l_1\rangle$, and $g_2 = \langle V_2 = \{w_1, w_2, w_3, w_4, w_5\}, E_2 = \{(w_1, w_2), (w_2, w_3), (w_2, w_4), (w_3, w_5)\}, l_2\rangle$. The left hand side of the figure shows the set of labels $L$, and the partial-order $\preceq$. In this case, $g_1 \not \sqsubseteq g_2$, but $g_1 \sqsubseteq_\preceq g_2$ via the mapping $m(v_1) = w_2$, $m(v_2) = w_3$ and $m(v_3) = w_4$, since the label $b$ is more general than the label $c$, and thus $v_2$ can be mapped to $w_3$.}
    \label{fig:subsumption1}
\end{figure}

Intuitively, the {\em subgraph} relationship is satisfied when one graph contains a subset of the edges and vertices of another one. The {\em subsumption} relationship generalizes the subgraph relationship not requiring the vertices and edges to be the same exact ones, but just that there is a mapping through which a given graph can be turned into a subgraph of the other. {\em Subsumption relative to $\preceq$} is an even more general relation when the subgraph is not required to actually have the same labels as the vertices in the supergraph, but just have labels that are smaller according to $\preceq$ than in the supergraph. The intuition behind this is that if labels represent concepts such as {\em vehicle} and {\em car}, and the $\preceq$ relation captures concept generality ({\em vehicle} $\preceq$ {\em car}), then the subgraph can have concepts that are more general or equal to those of the supergraph (i.e., a vertex labeled as {\em vehicle} in the subgraph can be mapped to a vertex labeled as {\em car} in the supergraph, since {\em vehicle} is a more general concept than {\em car}). The difference between subsumption and subsumption relative to $\preceq$ is illustrated in Figure \ref{fig:subsumption1}.

\begin{figure}[tb]
    \centering
    \includegraphics[width=0.75\columnwidth]{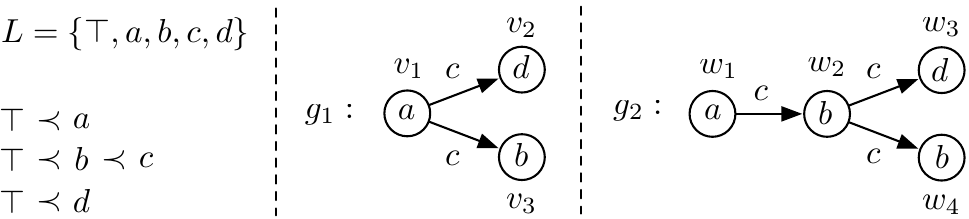}
    \caption{{\bf Subsumption} vs {\bf Trans-subsumption}: Two graphs, $g_1 = \langle V_1 = \{v_1, v_2, v_3\}, E_1 = \{(v_1, v_2), (v_2, v_3)\}, l_1\rangle$, and $g_2 = \langle V_2 = \{w_1, w_2, w_3, w_4\}, E_2 = \{(w_1, w_2), (w_2, w_3), (w_2, w_4)\}, l_2\rangle$. The left hand side of the figure shows the set of labels $L$, and the partial-order $\preceq$. In this case, $g_1 \not \sqsubseteq g_2$, but $g_1 \protect\overrightarrow{\sqsubseteq} g_2$ via the mapping $m(v_1) = w_1$, $m(v_2) = w_3$ and $m(v_3) = w_4$, the intuition here is that if label $c$ represents the concept ``descendant'', then $g_1$ is an individual $a$ with two descendants, one with label $b$, and one $d$, which is also the case in $g_2$, and thus $g_1$ trans-subsumes $g_2$.}
    \label{fig:subsumption2}
\end{figure}

{\em Trans-subsumption} generalizes the concept of subsumption by allowing an edge in the subgraph to span a chain of edges in the supergraph, as long as they all have the same label. This is useful when edges represent relations in the graph that are transitive (such as {\em descendant}). Figure \ref{fig:subsumption2} illustrates the difference between regular subsumption and trans-subsumption. Finally, {\em Trans-subsumption Relative to $\preceq$} is the most lax relation between graphs, combining trans-subsumption with a label partial order $\preceq$. In the remainder of this document, we will some times use the term ``regular subsumption'' to refer to subsumption as defined by Definitions \ref{def:subsumption} or \ref{def:subsumption-po}, as for distinguishing it from trans-subsumption.

Given a subsumption relation $\sqsubseteq$, we say that two graphs are equivalent, $g_1 \equiv g_2$ if $g_1 \sqsubseteq g_2$ and $g_1 \sqsubseteq g_2$. Moreover, if $g_1 \sqsubseteq g_2$ but $g_2 \not\sqsubseteq g_1$, then we write $g_1 \sqsubset g_2$.

\subsubsection{Object Identity}
An additional concept related to subsumption is that of {\em object identity} (OI) \cite{ferilli02aiiaa}, which is an additional constraint on the mapping $m$ employed for subsumption. The intuition behind object identity is that ``objects denoted with different symbols must be distinct''. 

When applied to subsumption ($\sqsubseteq$ or $\sqsubseteq_{\preceq}$) over graphs this translates to an additional constraint over the mapping $m: V_1 \to V_2$, namely that $v_1 \neq v_2 \implies m(v_1) \neq m(v_2)$.

When adding object identity to trans-subsumption ($\overrightarrow{\sqsubseteq}$ or $\overrightarrow{\sqsubseteq}_{\preceq}$), object identity translates to a slightly more elaborate constraint over the mappings $m$, and $m_e$. Specifically, we need to ensure that:
\begin{itemize}
\item $v_1 \neq v_2 \implies m(v_1) \neq m(v_2)$
\item $\forall v_1 \in V_1, (v_2, v_3) \in E_1$, if $v_1 \neq v_2$ and $v_1 \neq v_3$ then $\nexists e = (w_1, w_2) \in m_e((v_2,v_3))$ s.t. $m(v_1) = w_1 \vee m(v_1) = w_2$ (i.e., when an edge in $g_1$ is mapped to a path in $g_2$, no other vertex in $g_1$ can be mapped to any of the vertices in that path).
\end{itemize}

\begin{figure}[tb]
    \centering
    \includegraphics[width=0.55\columnwidth]{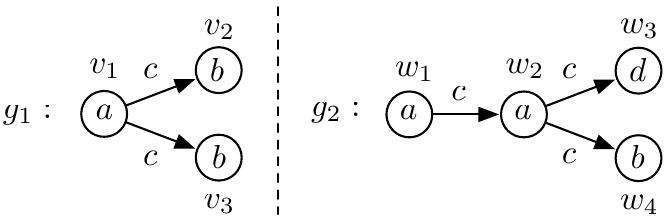}
    \caption{{\bf Object Identity}: If we enforce object identity $g_1 \not \sqsubseteq g_2$. However, if we do {\bf not} enforce object identity, then $g_1 \sqsubseteq g_2$ via the mapping $m(v_1) = w_2$, $m(v_2) = w_4$ and $m(v_3) = w_4$. Here, since no object identity is enforced, two different vertices of $g_1$ ($v_2$ and $v_3$) have been mapped to the same vertex in $g_2$ ($w_4$).}
    \label{fig:subsumption3}
\end{figure}

\subsection{Unification and Anti-unification}

The subsumption relation induces two basic operations over graphs: {\em unification} and {\em anti-unification}.

\begin{definition}[Unification]\label{def:unification}
Given two graphs $g_1$, and $g_2$, and a subsumption relation $\sqsubseteq$ (which can be any of the ones defined above), $g$ is a {\em unifier} of $g_1$ and $g_2$ (we write $g = g_1 \sqcup g_2$) if:
\begin{itemize}
\item $g_1 \sqsubseteq g$,
\item $g_2 \sqsubseteq g$,
\item and $\nexists g' \sqsubset g: g_1 \sqsubseteq g' \,\, \wedge \,\, g_2 \sqsubseteq g'$.
\end{itemize}
\end{definition}

In other words, a unifier of two graphs is a most general graph that is subsumed by two other graphs. The analogous operation is that of {\em anti-unification}.

\begin{definition}[Anti-unification]
Given two graphs $g_1$, and $g_2$, and a subsumption relation $\sqsubseteq$ (which can be any of the ones defined above), $g$ is an {\em anti-unifier} of $g_1$ and $g_2$ (we write $g = g_1 \sqcap g_2$) if:
\begin{itemize}
\item $g \sqsubseteq g_1$,
\item $g \sqsubseteq g_2$,
\item and $\nexists g' \sqsupset g: g' \sqsubseteq g_1 \,\, \wedge \,\, g' \sqsubseteq g_2$.
\end{itemize}
\end{definition}

In other words, an anti-unifier of two graphs is the most specific graph that subsumes both of them. 

Moreover, we will use the term {\em unifier} to denote a graph that satisfies Definition \ref{def:unification}, and the term {\em unification} to refer to the operation (i.e., the algorithm) that generates one more more unifiers given two input terms. We will use the terms {\em anti-unifier} and {\em anti-unification} analogously. Finally, if we see the subsumption order as a partial order, unification and anti-unification correspond to the traditional {\em join} and {\em meet} operations in partial orders.

\subsection{Distance and Similarity}

In the remainder of this document, we will use the following definition of distance and similarity function:

\begin{definition}[Distance function]
A {\em distance function} $d : G \times G \to \mathbb{R}$ on a set $G$ is a function satisfying the following conditions. Given any $g_1$, $g_2$, and $g_3 \in G$:
\begin{itemize}
\item $d(g_1, g_2) \geq 0$
\item $d(g_1, g_2) = 0 \iff g_1 = g_2$
\item $d(g_1, g_2) = d(g_2, g_1)$
\item $d(g_1, g_3) \leq d(g_1, g_2) + d(g_2, g_3)$ (triangle inequality)
\end{itemize}
\end{definition}

Although no commonly accepted formal definition of similarity exists, for the rest of this document, we will use the following definition.

\begin{definition}[Similarity function]
A {\em similarity function} $s : G \times G \to \mathbb{R}$ on a set $G$ is a function satisfying the following conditions. Given any $g_1$, and $g_2 \in G$:
\begin{itemize}
\item $0 \leq s(g_1, g_2) \leq 1$
\item $s(g_1, g_2) = 1 \iff g_1 = g_2$
\item $s(g_1, g_2) = s(g_2, g_1)$
\end{itemize}
\end{definition}

Although this formal definition suffices for our purposes, we should also have in mind that intuitively a similarity function is the inverse of a distance function.

%------------------------------------------------------------------------
%------------------------------------------------------------------------
%------------------------------------------------------------------------

\section{Refinement Operators for Labeled Graphs}

\begin{definition}[Downward Refinement Operator]
A {\em downward refinement operator} over a quasi-ordered set $(G,\sqsubseteq)$ is a function $\rho : G \to 2^G$ such that $\forall g' \in \rho(g) : g \sqsubseteq g'$.
\end{definition}

\begin{definition}[Upward Refinement Operator]
An {\em upward refinement operator} over a quasi-ordered set $(G,\sqsubseteq)$ is a function $\rho : G \to 2^G$ such that $\forall g' \in \rho(g) : g' \sqsubseteq g$.
\end{definition}

In the context of this document, a downward refinement operator generates elements of $G$ which are ``more specific'' (the complementary notion of upward refinement operator, corresponds to functions that generate elements of $G$ which are ``more general''). Moreover, refinement operators might satisfy certain properties of interest:

\begin{itemize}
\item A refinement operator $\rho$ is {\em locally finite} if $\forall g \in G : \rho(g)$ is finite.
\item A downward refinement operator $\rho$ is {\em complete} if $\forall g_1, g_2 \in G | g_1 \sqsubset g_2 : g_1 \in \rho^*(g_2)$.
\item An upward refinement operator $\gamma$ is {\em complete} if $\forall g_1, g_2 \in G | g_1 \sqsubset g_2 : g_2 \in \gamma^*(g_1)$.
\item A refinement operator $\rho$ is {\em proper} if $\forall g_1, g_2 \in G $ $ g_2 \in \rho(g_1) \Rightarrow g_1 \not \equiv g_2$.
\end{itemize}
\noindent where $\rho^*$ means the {\em transitive closure} of a refinement operator. Intuitively, {\em locally finiteness} means that the refinement operator is computable, {\em completeness} means we can generate, by refinement of $a$, any element of $G$ related to a given element $g_1$ by the order relation $\sqsubseteq$, and {\em properness} means that a refinement operator does not generate elements which are equivalent to the element being refined. When a refinement operator is locally finite, complete and proper, we say that it is {\em ideal}. %Other interesting properties of refinement operators have been discussed in the literature, such as {\em minimality} \cite{badea1999dlrefinement}, but are not relevant for the purposes of this document.

Notice that all the subsumption relations presented above satisfy the {\em reflexive}\footnote{A graph trivially subsumes itself with the mapping $m(v) = v$.} and {\em transitive}\footnote{If a graph $g_1$ subsumes another graph $g_2$ through a mapping $m_1$, and $g_2$ subsumes another graph $g_3$ through a mapping $m_2$, it is trivial to check that $g_1$ subsumes $g_3$ via the mapping $m(v) = m_2(m_1(v))$.} properties. Therefore, the pair $(G, \sqsubseteq)$, where $G$ is the set of all DLGs given a set of labels $L$, and $\sqsubseteq$ is any of the subsumption relations defined above is a {\em quasi-ordered} set. Thus, this opens the door to defining refinement operators for DLGs. Intuitively, a downward refinement operator for DLGs will generate {\em refinements} of a given DLG by either adding vertices, edges, or by making some of the labels more specific, thus making the graph more specific.

\begin{table*}[tb]
    \centering
    \begin{tabular}{|c|c|c|c|c|c|c|} 
    \multicolumn{7}{c}{Operators for flat-labeled DLGs} \\ \hline
    {\em Operator} 	& {\em Direction} 	& {\em Order} 			& {\em Locally Finite} 	& {\em Complete} 	& {\em Proper}		& {\em Propositions} \\ \hline
    $\rho_f$		& downward			& $(G, \sqsubseteq)$	& \checkmark			& \checkmark		& only under OI		& Propositions \ref{prop:rho-f-finite-complete} and \ref{prop:rho-f-ideal}		\\
    $\rho_{tf}$		& downward			& $(G, \overrightarrow{\sqsubseteq})$	& \checkmark & \checkmark	& only under OI	& Propositions \ref{prop:rho-tf-finite-complete} and \ref{prop:rho-tf-ideal}			\\
    $\gamma_f$		& upward			& $(G, \sqsubseteq)$	& \checkmark			& \checkmark		& only under OI		& Propositions \ref{prop:rho-upward-f-finite-complete} and \ref{prop:rho-upward-f-ideal}		\\
    $\gamma_{tf}$	& upward			& $(G, \overrightarrow{\sqsubseteq})$	& \checkmark& \checkmark	& only under OI	& Propositions \ref{prop:rho-upward-tf-finite-complete} and \ref{prop:rho-upward-tf-ideal}		\\ \hline

    \multicolumn{7}{c}{Operators for order-labeled DLGs} \\ \hline
    {\em Operator} 	& {\em Direction} 	& {\em Order} 			& {\em Locally Finite} 	& {\em Complete} 	& {\em Proper}		& {\em Propositions} \\ \hline
    $\rho_{\preceq}$	& downward		& $(G, \sqsubseteq_{\preceq})$	& \checkmark		& \checkmark	& only under OI	& Propositions \ref{prop:rho-prec-finite-complete} and \ref{prop:rho-prec-ideal}		\\
    $\rho_{t\preceq}$	& downward		& $(G, \overrightarrow{\sqsubseteq}_{\preceq})$	& \checkmark	& \checkmark 	& only under OI	& Propositions \ref{prop:rho-tprec-finite-complete} and \ref{prop:rho-tprec-ideal}		\\
    $\gamma_{\preceq}$	& upward		& $(G, \sqsubseteq_{\preceq})$	& \checkmark		& \checkmark	& only under OI	& Propositions \ref{prop:gamma-prec-finite-complete} and \ref{prop:gamma-prec-ideal}		\\
    $\gamma_{t\preceq}$	& upward		& $(G, \overrightarrow{\sqsubseteq}_{\preceq})$	& \checkmark	& \checkmark	& only under OI	& Propositions \ref{prop:gamma-tprec-finite-complete} and \ref{prop:gamma-tprec-ideal}		\\ \hline
    \end{tabular}
    \caption{Summary of the theoretical properties of the refinement operators presented in this document, together with the propositions where their properties are proven.}
    \label{tbl:operators}
\end{table*}

In the following subsections, we will introduce a collection of refinement operators for connected DLGs, and discuss their theoretical properties. A summary of these operators is shown in Table \ref{tbl:operators}, where we show that under the object-identity constraint, all the refinement operators presented in this document are ideal. If we do not impose object-identity, then the operators are locally complete and complete, but not proper.

\subsection{Downward Refinement of Flat-labeled DLGs}\label{sec:operators-downward-flat}

We will define refinement operators as sets of {\em rewriting rules}. A rewriting rule is composed of three parts: the applicability conditions (shown between square brackets), the original graph (above the line), and the refined graph (below the line). Given a DLG $g = \langle V, E, l \rangle$, the following rewriting rules define two downward refinement operators $\rho_f$, and $\rho_{tf}$ for flat-labeled DLGs (Figure \ref{fig:rewrite-rules1} shows examples of the application of each of the rewrite rules):

\begin{description}

\item[(R0)] Top operator (adds one vertex to an empty graph):
\begin{eqnarray*}
\left[ 
\begin{array}{l}
v_*\not\in V, \\
V = \emptyset, \\
E = \emptyset, \\
a \in L \\
\end{array}
\right] & \,\, & 
\cfrac{ \langle V, E, l \rangle }
	  { \left\langle V \cup \{ v_* \}, E, l'(x) = \left\{ 
	  		\begin{array}{ll} 
			a & \text{if} \,\, x = v_* \\  
			l(x) & \text{otherwise}
			\end{array} 
			\right.
	    \right\rangle }
\end{eqnarray*}

\item[(R1)] Add vertex operator with outgoing edge:
\begin{eqnarray*}
\left[ 
\begin{array}{l}
v_*\not\in V, \\
v_1 \in V, \\
a \in L, \\
b \in L
\end{array}
\right] & \,\, & 
\cfrac{ \langle V, E, l \rangle }
	  { \left\langle V \cup \{ v_* \}, E \cup \{ (v_*, v_1) \}, l'(x) = \left\{ 
	  		\begin{array}{ll} 
			a & \text{if} \,\, x = v_* \\  
			b & \text{if} \,\, x = (v_*, v_1) \\  
			l(x) & \text{otherwise}
			\end{array} 
			\right.
	    \right\rangle }
\end{eqnarray*}

\item[(R2)] Add vertex operator with incoming edge:
\begin{eqnarray*}
\left[ 
\begin{array}{l}
v_*\not\in V, \\
v_1 \in V, \\
a \in L, \\
b \in L
\end{array}
\right] & \,\, & 
\cfrac{ \langle V, E, l \rangle }
	  { \left\langle V \cup \{ v_* \}, E \cup \{ (v_1, v_*) \}, l'(x) = \left\{ 
	  		\begin{array}{ll} 
			a & \text{if} \,\, x = v_* \\  
			b & \text{if} \,\, x = (v_1, v_*) \\  
			l(x) & \text{otherwise}
			\end{array} 
			\right.
	    \right\rangle }
\end{eqnarray*}

\item[(R3)] Add edge operator:
\begin{eqnarray*}
\left[ 
\begin{array}{l}
v_1 \in V, \\
v_2 \in V, \\
(v_1, v_2) \not\in E, \\
a \in L
\end{array}
\right] & \,\, & 
\cfrac{ \langle V, E, l \rangle }
	  { \left\langle V, E \cup \{ (v_1, v_2) \}, l'(x) = \left\{ 
	  		\begin{array}{ll} 
			a & \text{if} \,\, x = (v_1, v_2) \\  
			l(x) & \text{otherwise}
			\end{array} 
			\right.
	    \right\rangle }
\end{eqnarray*}

\item[(R4)] Split edge operator:
\begin{eqnarray*}
\left[ 
\begin{array}{l}
v_*\not\in V, \\
(v_1, v_2) \in E, \\
a \in L, \\
b = l((v_1, v_2))
\end{array}
\right] & \,\, & 
\cfrac{ \langle V, E, l \rangle }
	  { \left\langle V \cup \{ v_* \}, \left( E \cup \{ (v_1, v_*), (v_*, v_2) \} \right) \setminus \{(v_1, v_2)\}, l'(x) = \left\{ 
	  		\begin{array}{ll} 
			a & \text{if} \,\, x = v_* \\  
			b & \text{if} \,\, x = (v_1, v_*) \\  
			b & \text{if} \,\, x = (v_*, v_2) \\  
			l(x) & \text{otherwise}
			\end{array} 
			\right.
	    \right\rangle }
\end{eqnarray*}

\end{description}

\begin{figure}[tb]
    \centering
    \includegraphics[width=0.9\columnwidth]{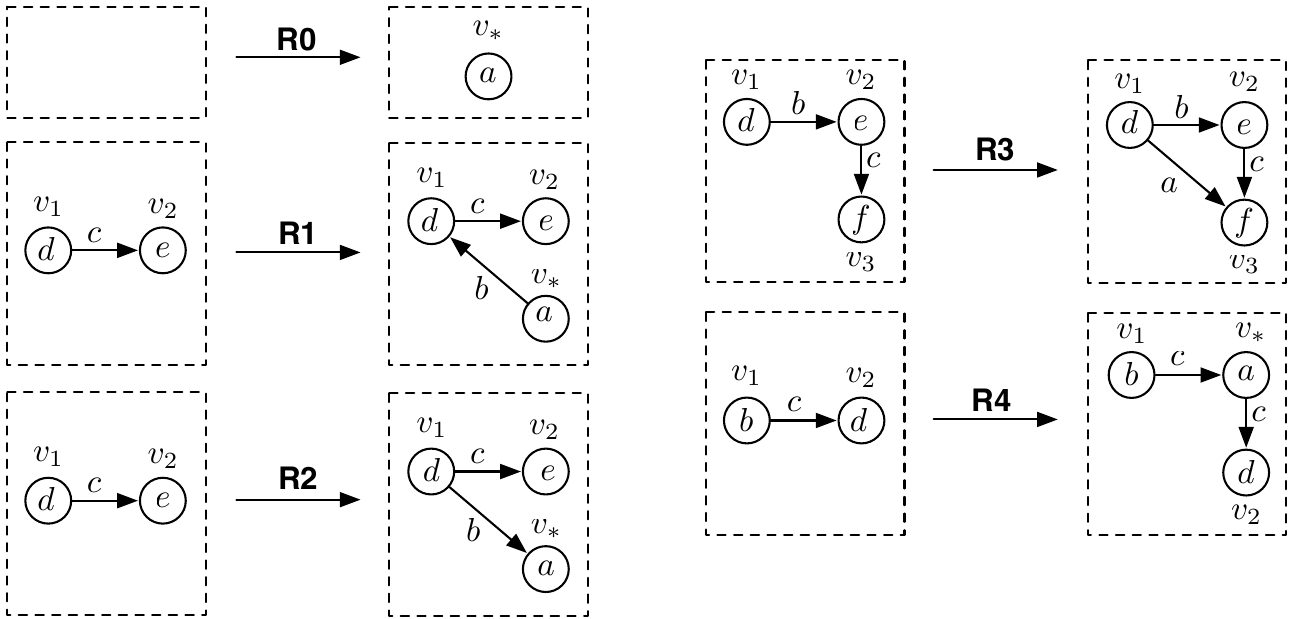}
    \caption{Example Application of the Rewrite Rules R0, R1, R2, R3 and R4.}
    \label{fig:rewrite-rules1}
\end{figure}

\begin{proposition}\label{prop:rho-f-finite-complete}
The downward refinement operator $\rho_f$ defined by the rewrite rules R0, R1, R2, and R3 above is locally finite, and complete for the quasi-ordered set $\langle G, \sqsubseteq \rangle$ (where $\sqsubseteq$ represents regular subsumption).
\end{proposition}

\begin{proposition}\label{prop:rho-f-ideal}
The downward refinement operator $\rho_f$ defined by the rewrite rules R0, R1, R2, and R3 above is ideal (locally finite, complete, and proper) for the quasi-ordered set $\langle G, \sqsubseteq \rangle$ (where $\sqsubseteq$ represents regular subsumption), when we impose the Object Identity constraint.
\end{proposition}

\begin{proposition}\label{prop:rho-tf-finite-complete}
The downward refinement operator $\rho_{tf}$ defined by the rewrite rules R0, R1, R2, R3, and R4 above is locally finite, and complete for the quasi-ordered set $\langle G, \overrightarrow{\sqsubseteq} \rangle$ (where $\overrightarrow{\sqsubseteq}$ represents trans-subsumption).
\end{proposition}

\begin{proposition}\label{prop:rho-tf-ideal}
The downward refinement operator $\rho_{tf}$ defined by the rewrite rules R0, R1, R2, R3, and R4 above is ideal (locally finite, complete, and proper) for the quasi-ordered set $\langle G, \overrightarrow{\sqsubseteq} \rangle$ (where $\overrightarrow{\sqsubseteq}$ represents trans-subsumption), when we impose the Object Identity constraint.
\end{proposition}

%Notice that the refinement operator $\rho_{tf}$ defined by the two rewrite rules R0, R1, R2, R3, and R4 above is, however, not ideal, since it is not proper. This is because refinements generated by the {\em split edge operator} R4 might actually subsume the original graph. This refinement is necessary, however, for making $\rho_{tf}$ complete under trans-subsumption.

Proofs to all these propositions can be found in Appendix \ref{appendix:proofs}.

\subsection{Upward Refinement of Flat-labeled DLGs}\label{sec:operators-upward-flat}

The following rewriting rules define two upward refinement operators, $\gamma_f$ and $\gamma_{tf}$, for flat-labeled DLGs:

\begin{description}

\item[(UR0)] Remove Bridge (removes a non-bridge edge of the graph, see Definition \ref{def:bridge}):
\begin{eqnarray*}
\left[ 
\begin{array}{l}
e \in E, \\
e \not\in \mathit{bridges(\langle V, E, l \rangle)}  \\
\end{array}
\right] & \,\, & 
\cfrac{ \langle V, E, l \rangle }
	  { \left\langle V, E \setminus \{e\}, l \right\rangle }
\end{eqnarray*}

\item[(UR1)] Remove Leaf (removes vertex connected to the rest of the graph by at most a single edge):
\begin{eqnarray*}
\left[ 
\begin{array}{l}
v \in V, \\
E_v = \{ e = (v_1, v_2) \in E \,\, | \,\, v = v_1 \,\, \vee \,\, v = v_2 \}, \\
|E_v| \leq 1, \\
\end{array}
\right] & \,\, & 
\cfrac{ \langle V, E, l \rangle }
	  { \left\langle V \setminus \{v\}, E \setminus E_v, l \right\rangle }
\end{eqnarray*}

Note: notice that here, $E_v$ is basically the set of edges that involve $v$, and by enforcing $|E_v| \leq 1$, we are basically selecting only those vertices $v \in V$ that are either: a) connected to the rest of the graph by at most one single edge (when $|E_v| = 1$), or b) when the graph is just composed of a single vertex and no edges (when $|E_v| = 0$), in which case, this operator turns the graph into $g_\top$.

\item[(UR2)] Shorten Edge (inverse of the R4 operator defined above):
\begin{eqnarray*}
\left[ 
\begin{array}{l}
e_1 = (v_1, v_2) \in E, \\
e_2 = (v_2, v_3) \in E, \\
l(e_1) = l(e_2), \\
\nexists e = (w_1,w_2) \in E | e \neq e_1 \,\, \wedge \,\, w_2 = v_2, \\
\nexists e = (w_1,w_2) \in E | e \neq e_2 \,\, \wedge \,\, w_1 = v_2
\end{array}
\right] & \,\, & 
\cfrac{ \langle V, E, l \rangle }
	  { \left\langle V \setminus \{v_2\}, \left( E \cup \{(v_1,v_3)\} \right) \setminus \{e_1, e_2\}, 
	  l'(x) \right\rangle }
\end{eqnarray*}
\noindent where
\[ l'(x) = \left\{ 
	  		\begin{array}{ll} 
			l(e_1) & \text{if} \,\, x = (v_1,v_3) \\  
			l(x) & \text{otherwise}
			\end{array} 
			\right.
\]

\end{description}

\begin{proposition}\label{prop:rho-upward-f-finite-complete}
The upward refinement operator $\gamma_f$ defined by the rewrite rules UR0 and UR1 above is locally finite, and complete for the quasi-ordered set $\langle G, \sqsubseteq \rangle$ (where $\sqsubseteq$ represents regular subsumption).
\end{proposition}

\begin{proposition}\label{prop:rho-upward-f-ideal}
The upward refinement operator $\gamma_f$ defined by the rewrite rules UR0 and UR1 above is ideal (locally finite, complete and proper) for the quasi-ordered set $\langle G, \sqsubseteq \rangle$ (where $\sqsubseteq$ represents regular subsumption), when we impose the Object Identity constraint.
\end{proposition}

\begin{proposition}\label{prop:rho-upward-tf-finite-complete}
The upward refinement operator $\gamma_{tf}$ defined by the rewrite rules UR0, UR1 and UR2 above is locally finite, and complete for the quasi-ordered set $\langle G, \overrightarrow{\sqsubseteq} \rangle$ (where $\overrightarrow{\sqsubseteq}$ represents trans-subsumption).
\end{proposition}

\begin{proposition}\label{prop:rho-upward-tf-ideal}
The upward refinement operator $\gamma_{tf}$ defined by the rewrite rules UR0, UR1 and UR2 above is ideal (locally finite, complete and proper) for the quasi-ordered set $\langle G, \overrightarrow{\sqsubseteq} \rangle$ (where $\overrightarrow{\sqsubseteq}$ represents trans-subsumption), when we impose the Object Identity constraint.
\end{proposition}

Proofs to all these propositions can be found in Appendix \ref{appendix:proofs}.

\subsection{Downward Refinement of Order-labeled DLGs}\label{sec:operators-downward-po}

Assuming the set $\langle L, \preceq \rangle$ is a partial order with a top element $\top \in L$, and given a DLG $g = \langle V, E, l \rangle$, the following rewriting rules define two downward refinement operators, $\rho_{\preceq}$ and $\rho_{t\preceq}$, for order-labeled DLGs:

\begin{description}
\item[(R0PO)] Top operator (adds one vertex to an empty graph):
\begin{eqnarray*}
\left[ 
\begin{array}{l}
v_* \not\in V, \\
V = \emptyset, \\
E = \emptyset
\end{array}
\right] & \,\, & 
\cfrac{ \langle V, E, l \rangle }
	  { \left\langle V \cup \{ v_* \}, E, l'(x) = \left\{ 
	  		\begin{array}{ll} 
			\top & \text{if} \,\, x = v_* \\  
			l(x) & \text{otherwise}
			\end{array} 
			\right.
	    \right\rangle }
\end{eqnarray*}

\item[(R1PO)] Add vertex operator with outgoing edge:
\begin{eqnarray*}
\left[ 
\begin{array}{l}
v_* \not\in V, \\
v_1 \in V
\end{array}
\right] & \,\, & 
\cfrac{ \langle V, E, l \rangle }
	  { \left\langle V \cup \{ v_* \}, E \cup \{ (v_1, v_*) \}, l'(x) = \left\{ 
	  		\begin{array}{ll} 
			\top & \text{if} \,\, x = v_* \\  
			\top & \text{if} \,\, x = (v_1, v_*) \\  
			l(x) & \text{otherwise}
			\end{array} 
			\right.
	    \right\rangle }
\end{eqnarray*}

\item[(R2PO)] Add vertex operator with incoming edge:
\begin{eqnarray*}
\left[ 
\begin{array}{l}
v_* \not\in V, \\
v_1 \in V
\end{array}
\right] & \,\, & 
\cfrac{ \langle V, E, l \rangle }
	  { \left\langle V \cup \{ v_* \}, E \cup \{ (v_*, v_1) \}, l'(x) = \left\{ 
	  		\begin{array}{ll} 
			\top & \text{if} \,\, x = v_* \\  
			\top & \text{if} \,\, x = (v_*, v_1) \\  
			l(x) & \text{otherwise}
			\end{array} 
			\right.
	    \right\rangle }
\end{eqnarray*}

\item[(R3PO)] Add edge operator:
\begin{eqnarray*}
\left[ 
\begin{array}{l}
v_1 \in V, \\
v_2 \in V, \\
(v_1, v_2) \not\in E \\
\end{array}
\right] & \,\, & 
\cfrac{ \langle V, E, l \rangle }
	  { \left\langle V, E \cup \{ (v_1, v_2) \}, l'(x) = \left\{ 
	  		\begin{array}{ll} 
			\top & \text{if} \,\, x = (v_1, v_2) \\  
			l(x) & \text{otherwise}
			\end{array} 
			\right.
	    \right\rangle }
\end{eqnarray*}

\item[(R4PO)] Refine vertex label:
\begin{eqnarray*}
\left[ 
\begin{array}{l}
v_1 \in V, \\
a = l(v_1), \\
b \in L, \\
a \prec b, \\
\nexists c\in L: a \preceq c \preceq b 
\end{array}
\right] & \,\, & 
\cfrac{ \langle V, E, l \rangle }
	  { \left\langle V, E, l'(x) = \left\{ 
	  		\begin{array}{ll} 
			b & \text{if} \,\, x = v_1 \\  
			l(x) & \text{otherwise}
			\end{array} 
			\right.
	    \right\rangle }
\end{eqnarray*}

\item[(R5PO)] Refine edge label (relative to $\preceq$):
\begin{eqnarray*}
\left[ 
\begin{array}{l}
e \in E, \\
a = l(e), \\
b \in L, \\
a \prec b, \\
\nexists c\in L: a \preceq c \preceq b 
\end{array}
\right] & \,\, & 
\cfrac{ \langle V, E, l \rangle }
	  { \left\langle V, E, l'(x) = \left\{ 
	  		\begin{array}{ll} 
			b & \text{if} \,\, x = e \\  
			l(x) & \text{otherwise}
			\end{array} 
			\right.
	    \right\rangle }
\end{eqnarray*}

\item[(R6PO)] Split edge operator:
\begin{eqnarray*}
\left[ 
\begin{array}{l}
v_*\not\in V, \\
(v_1, v_2) \in E, \\
b = l((v_1, v_2))
\end{array}
\right] & \,\, & 
\cfrac{ \langle V, E, l \rangle }
	  { \left\langle V \cup \{ v_* \}, \left( E \cup \{ (v_1, v_*), (v_*, v_2) \} \right) \setminus \{(v_1, v_2)\}, l'(x) = \left\{ 
	  		\begin{array}{ll} 
			\top & \text{if} \,\, x = v_* \\  
			b & \text{if} \,\, x = (v_1, v_*) \\  
			b & \text{if} \,\, x = (v_*, v_2) \\  
			l(x) & \text{otherwise}
			\end{array} 
			\right.
	    \right\rangle }
\end{eqnarray*}

\end{description}

\begin{proposition}\label{prop:rho-prec-finite-complete}
The downward refinement operator $\rho_{\preceq}$ defined by the rewrite rules R0PO, R1PO, R2PO, R3PO, R4PO, and R5PO above is locally finite and complete for the quasi-ordered set $\langle G, \sqsubseteq_{\prec} \rangle$ (where $\sqsubseteq_{\prec}$ represents subsumption relative to the partial order $\prec$).
\end{proposition}

\begin{proposition}\label{prop:rho-prec-ideal}
The downward refinement operator  $\rho_{\preceq}$ defined by the rewrite rules R0PO, R1PO, R2PO, R3PO, R4PO and R5PO above is ideal (locally finite, complete, and proper) for the quasi-ordered set $\langle G, \sqsubseteq_{\prec} \rangle$ (where $\sqsubseteq_{\prec}$ represents subsumption relative to the partial order $\prec$), when we impose the Object Identity constraint.
\end{proposition}

\begin{proposition}\label{prop:rho-tprec-finite-complete}
The downward refinement operator $\rho_{t\preceq}$ defined by the rewrite rules R0PO, R1PO, R2PO, R3PO, R4PO, R5PO, and R6PO above is locally finite, and complete for the quasi-ordered set $\langle G, \overrightarrow{\sqsubseteq}_{\prec} \rangle$ (where $\overrightarrow{\sqsubseteq}_{\prec}$ represents trans-subsumption relative to the partial order $\prec$).
\end{proposition}

\begin{proposition}\label{prop:rho-tprec-ideal}
The downward refinement operator  $\rho_{t\preceq}$ defined by the rewrite rules R0PO, R1PO, R2PO, R3PO, R4PO, R5PO, and R6PO above is ideal (locally finite, complete, and proper) for the quasi-ordered set $\langle G, \overrightarrow{\sqsubseteq}_{\prec} \rangle$ (where $\overrightarrow{\sqsubseteq}_{\prec}$ represents trans-subsumption relative to the partial order $\prec$), when we impose the Object Identity constraint.
\end{proposition}

Proofs to all these propositions can be found in Appendix \ref{appendix:proofs}.

\subsection{Upward Refinement of Order-labeled DLGs}\label{sec:operators-upward-po}

The following rewriting rules define two upward refinement operators, $\gamma_{\preceq}$ and $\gamma_{t\preceq}$, for order-labeled DLGs:

\begin{description}
\item[(UR0PO)] Generalize vertex label:
\begin{eqnarray*}
\left[ 
\begin{array}{l}
v_1 \in V, \\
a = l(v_1), \\
b \in L, \\
b \prec a, \\
\nexists c\in L: a \preceq c \preceq b 
\end{array}
\right] & \,\, & 
\cfrac{ \langle V, E, l \rangle }
	  { \left\langle V, E, l'(x) = \left\{ 
	  		\begin{array}{ll} 
			b & \text{if} \,\, x = v_1 \\  
			l(x) & \text{otherwise}
			\end{array} 
			\right.
	    \right\rangle }
\end{eqnarray*}

\item[(UR1PO)] Generalize edge label:
\begin{eqnarray*}
\left[ 
\begin{array}{l}
e \in E, \\
a = l(e), \\
b \in L, \\
b \prec a, \\
\nexists c\in L: a \preceq c \preceq b 
\end{array}
\right] & \,\, & 
\cfrac{ \langle V, E, l \rangle }
	  { \left\langle V, E, l'(x) = \left\{ 
	  		\begin{array}{ll} 
			b & \text{if} \,\, x = e \\  
			l(x) & \text{otherwise}
			\end{array} 
			\right.
	    \right\rangle }
\end{eqnarray*}

\item[(UR2PO)] Remove Bridge (removes a top non-bridge edge of the graph, see Definition \ref{def:bridge}):
\begin{eqnarray*}
\left[ 
\begin{array}{l}
e \in E, \\
l(e) = \top, \\
e \not\in \mathit{bridges(\langle V, E, l \rangle)}  \\
\end{array}
\right] & \,\, & 
\cfrac{ \langle V, E, l \rangle }
	  { \left\langle V, E \setminus \{e\}, l \right\rangle }
\end{eqnarray*}

\item[(UR3PO)] Remove Leaf (removes a top vertex connected to the rest of the graph by at most a single edge):
\begin{eqnarray*}
\left[ 
\begin{array}{l}
v \in V, \\
l(v) = \top, \\
E_v = \{ e = (v_1, v_2) \in E \,\, | \,\, v = v_1 \,\, \vee \,\, v = v_2 \}, \\
|E_v| \leq 1, \\
\forall e \in E_v : l(e) = \top \\
\end{array}
\right] & \,\, & 
\cfrac{ \langle V, E, l \rangle }
	  { \left\langle V \setminus \{v\}, E \setminus E_v, l \right\rangle }
\end{eqnarray*}

Note: notice that here, $E_v$ is basically the set of edges that involve $v$, and by enforcing $|E_v| \leq 1$, we are basically selecting only those vertices $v \in V$ that are either: a) connected to the rest of the graph by at most one single edge (when $|E_v| = 1$), or b) when the graph is just composed of a single vertex and no edges (when $|E_v| = 0$), in which case, this operator turns the graph into $g_\top$.

\item[(UR4PO)] Shorten Edge (inverse of the R6PO operator defined above):
\begin{eqnarray*}
\left[ 
\begin{array}{l}
e_1 = (v_1, v_2) \in E, \\
e_2 = (v_2, v_3) \in E, \\
l(v_2) = \top, \\
l(e_1) = l(e_2), \\
\nexists e = (w_1,w_2) \in E | e \neq e_1 \,\, \wedge \,\, w_2 = v_2, \\
\nexists e = (w_1,w_2) \in E | e \neq e_2 \,\, \wedge \,\, w_1 = v_2
\end{array}
\right] & \,\, & 
\cfrac{ \langle V, E, l \rangle }
	  { \left\langle V \setminus \{v_2\}, \left( E \cup \{(v_1,v_3)\} \right) \setminus \{e_1, e_2\}, 
	  l'(x) \right\rangle }
\end{eqnarray*}
\noindent where
\[ l'(x) = \left\{ 
	  		\begin{array}{ll} 
			l(e_1) & \text{if} \,\, x = (v_1,v_3) \\  
			l(x) & \text{otherwise}
			\end{array} 
			\right.
\]
\end{description}

\begin{proposition}\label{prop:gamma-prec-finite-complete}
The upward refinement operator $\gamma_{\preceq}$ defined by the rewrite rules UR0PO, UR1PO, UR2PO, and UR3PO above is locally finite and complete for the quasi-ordered set $\langle G, \sqsubseteq_{\prec} \rangle$ (where $\sqsubseteq_{\prec}$ represents subsumption relative to the partial order $\prec$).
\end{proposition}

\begin{proposition}\label{prop:gamma-prec-ideal}
The upward refinement operator  $\gamma_{\preceq}$ defined by the rewrite rules UR0PO, UR1PO, UR2PO, and UR3PO above is ideal (locally finite, complete, and proper) for the quasi-ordered set $\langle G, \sqsubseteq_{\prec} \rangle$ (where $\sqsubseteq_{\prec}$ represents subsumption relative to the partial order $\prec$), when we impose the Object Identity constraint.
\end{proposition}

\begin{proposition}\label{prop:gamma-tprec-finite-complete}
The upward refinement operator $\gamma_{t\preceq}$ defined by the rewrite rules UR0PO, UR1PO, UR2PO, UR3PO, and UR4PO above is locally finite, and complete for the quasi-ordered set $\langle G, \overrightarrow{\sqsubseteq}_{\prec} \rangle$ (where $\overrightarrow{\sqsubseteq}_{\prec}$ represents trans-subsumption relative to the partial order $\prec$).
\end{proposition}

\begin{proposition}\label{prop:gamma-tprec-ideal}
The upward refinement operator  $\gamma_{t\preceq}$ defined by the rewrite rules UR0PO, UR1PO, UR2PO, UR3PO, and UR4PO above is ideal (locally finite, complete, and proper) for the quasi-ordered set $\langle G, \overrightarrow{\sqsubseteq}_{\prec} \rangle$ (where $\overrightarrow{\sqsubseteq}_{\prec}$ represents trans-subsumption relative to the partial order $\prec$), when we impose the Object Identity constraint.
\end{proposition}

Proofs to all these propositions can be found in Appendix \ref{appendix:proofs}.

\subsection{Trees}

In many domains of interest, data of interest can be represented using trees, which are more restricted than full-fledged DLGs. If we know our graphs are actually trees, refinement operators can be defined that exploit this fact to improve computational efficiency (and the implementation of the subsumption operations defined above can be made more efficient). We will not provide any additional theoretical results for the case of trees, since, in theory, given a refinement operator for general DLGs, it is trivial to turn it into a refinement operator for trees by filtering out all the refinements that are not trees. However, for the sake of efficiency, \rhog\ implements special versions for trees of all the refinement operators, and subsumption relations, which are significantly more efficient.

\subsection{Refinement Graphs}

Refinement operators allow us to define a series of concepts relevant for similarity assessment.

\begin{definition}[Refinement Graph]
Given a downward refinement operator $\rho$, and the (infinite) set $G$ of all possible directed labeled graphs that can be constructed with a given set of labels $L$, we define a refinement graph $\langle G, \rho \rangle$, as a graph where each graph $g \in G$ is a vertex, and there is an edge between $g_1$ and $g_2$ if $g_2$ is a downward refinement of $r_1$ according to $\rho$ ($g_2 \in \rho(g_1)$).
\end{definition}

Notice that the refinement graph is naturally structured as a partial order, whose top element is the graph $g_\top = \langle \{\}, \{\}, l \rangle$ (which subsumes any other graph in $G$).

\begin{definition}[Refinement Path]
A finite sequence of graphs $[g_1, ..., g_n]$ is a {\em refinement path} $g_1 \xrightarrow{\rho} g_n$ between $g_1$ and $g_n$ when for each $1 \leq i < n$, $g_{i+1} \in \rho(g_{i})$.
\end{definition}

We will write $|g_1 \xrightarrow{\rho} g_2|$ to denote the length of the shortest refinement path between $g_1$ and $g_2$, where the length is measured as the number of times that a refinement operator needs to be applied to $g_1$ to reach $g_2$ ($|g_1 \xrightarrow{\rho} g_1| = 0$).

%------------------------------------------------------------------------
%------------------------------------------------------------------------
%------------------------------------------------------------------------

\section{Refinement-based Similarity Measures Between Graphs}

Graph subsumption introduces a concept of {\em information order} between graphs: if a graph $g_1$ subsumes another graph $g_2$, then all the information in $g_1$ is also in $g_2$. Thus, if we find the most specific graph $g$ that subsumes two other graphs $g_1$, and $g_2$, then $g$ captures the information that $g_1$ and $g_2$ have in common. The intuition of the refinement-based similarity functions is to first compute such $g$, and then numerically quantify the amount of information in $g$, which is a measure of how similar $g_1$ and $g_2$ are: the more information they share, the more similar they are. In order to numerically measure the amount of information in a graph, we will use the intuition that each time we apply a downward refinement operation, we introduce one new piece of information, so the length of the refinement path between $g_\top$ and a given graph $g$ gives us a measure of the amount of information contained in it.

In our previous work \cite{ontanon2009similarity,ontanon2012similarity}, we introduced the concept of refinement-based similarity measures in the context of feature-terms \cite{carpenterbook} (a representation formalism used in structured machine learning and in natural language processing), and later extend this idea to other formalisms such as Description Logics \cite{sanchez2013refinement}, and partial-order plans \cite{sanchez2014least}. Here, we extend these ideas further to directed labeled graphs (DLGs) by using the different subsumption relations and refinement operators introduced above. 

\subsection{Anti-unification-based Similarity}

\begin{definition}[Anti-unification-based Similarity]
Given two graphs $g_1$, and $g_2$, a refinement operator $\rho$ and a subsumption relation $\sqsubseteq$, the {\em anti-unification-based similarity} $S_\lambda$ is defined as:
\[S_\lambda(g_1, g_2) = \frac{|g_\top \xrightarrow{\rho} (g_1 \sqcap g_2)|}
					         {|g_\top \xrightarrow{\rho} (g_1 \sqcap g_2)| + 
					          |(g_1 \sqcap g_2) \xrightarrow{\rho} g_1| + 
					          |(g_1 \sqcap g_2) \xrightarrow{\rho} g_2|}\]
\end{definition}

Intuitively, this measures the amount of information shared between $g_1$ and $g_2$ (size of their anti-unifier), and normalizes it by the total amount of information: shared information ($|g_\top \xrightarrow{\rho} (g_1 \sqcap g_2)|$), information in $g_1$ but not in $g_2$ ($|(g_1 \sqcap g_2) \xrightarrow{\rho} g_1|$, and information in $g_2$ but not in $g_1$ ($|(g_1 \sqcap g_2) \xrightarrow{\rho} g_2|$). The reader is referred to our previous work \cite{ontanon2012similarity} for a more in-depth description and analysis of the anti-unification-based similarity.

One interesting thing about $S_\lambda$ is that given a refinement operator, and a subsumption relation, it is applicable to any representation formalism: there is nothing specific to directed labeled graphs in this formulation. In other words, the refinement operator allows us to abstract away from the underlying representation formalism.

\subsection{Properties-based Similarity}

The key idea of the {\em properties-based similarity} measure is to decompose each graph into a collection of smaller graphs (which we will call {\em properties}), and then count how many of these properties are shared between two given graphs. The key advantage of this similarity measure is that each of these properties can be seen as a {\em feature}, and thus, we can apply feature weighting methods in order to improve accuracy in the context of machine learning methods. Let us first explain how to decompose a graph into a collection of properties (operation, which we call {\em disintegration} \cite{ontanon2012similarity}).

\subsubsection{Graph Disintegration}

Consider a refinement path $g_0 = [g_\top, ..., g_n]$ between the most general graph $g_\top$ and a given graph $g_n$, generated by repeated application of either a downward refinement operator (going from $g_\top$ to $g_n$), or by repeated application of an upward refinement operator (going from $g_n$ to $g_\top$). The intuition of graph disintegration is the following: each time an upward refinement operator is applied to a graph $g_{i+1}$ to generate a more general graph $g_{i}$, a {\em piece of information} is removed, which $g_i$ does not have, and $g_{i+1}$ had: so, each step in the refinement path removes a piece of information (when moving from $g_n$ to $g_\top$). We would like the {\em disintegration} operation to decompose graph $g_n$ into exactly $n$ properties, each of them representing each of the pieces of information that were removed along the refinement path. 

In order to do this, we will introduce the concept of the {\em remainder} operation (introduced by Onta\~{n}\'{o}n and Plaza \cite{ontanon2012similarity} for feature terms):

\begin{definition}[Remainder]
Given two graphs $g_u$ and $g_d$ such that $g_u \sqsubseteq g_d$, the remainder $r(g_d,g_u)$ is a graph $g_r$ such that $g_r \sqcup g_u \equiv g_d$, and $\nexists g \in G$ such that $g \sqsubset g_r$ and $g \sqcup g_u \equiv g_d$.
\end{definition}

In other words, the remainder is the most general graph $g_r$ such that when unifying $g_r$ with the most general of the two graphs ($g_u$), recovers the most specific of the two graphs ($g_d$). Now, given two graphs $g_i$, and $g_{i+1}$, such that $g_i$ is an upward refinement of $g_{i+1}$ via the upward refinement operator $\gamma$, $r(g_{i+1},g_i)$ is precisely the graph that captures the piece of information that $\gamma$ ``removed'' from $g_{i+1}$. Moreover, it is possible to compute the remainder of a generalization operation without the need to actually perform any type of unification operation, which can be computationally expensive. 

\begin{algorithm}[tb] \caption{Remainder: $r(g_u, g_d, \gamma)$}\label{alg:remainder}
\begin{algorithmic}[1]
\medskip
\STATE $t := 0, \pi_0 := g_d$
\STATE $A = \{g \in \gamma(g_d) | g_u \sqsubseteq g\}$
\WHILE {$\mathit(true)$}
%\STATE $N = \{g \in \gamma(\pi_t) \vert g \not\sqsubseteq g_u \; \wedge  \; g \sqcup \psi_u \equiv g_d\}$
\STATE $N = \{g \in \gamma(\pi_t) \vert g \not\sqsubseteq g_u \; \wedge  \nexists g' \in A | g \sqsubseteq g'\}$
\IF {$N = \emptyset$}
\RETURN $\pi_t$
\ENDIF
\STATE $\pi_{t+1}$ selected stochastically from $N$ 
\STATE $t := t + 1$
\ENDWHILE
\end{algorithmic}
\end{algorithm}

The algorithm for computing the remainder can be shown in Algorithm \ref{alg:remainder} (which is an optimized version of the algorithm presented by Onta\~{n}\'{o}n and Plaza \cite{ontanon2015refinement}), and works as follows: the algorithm runs for a series of iterations. At each iteration $t$, the algorithm starts with a candidate graph $\pi_t$, which is ensured to satisfy $\pi_t \sqcup g_u \equiv g_d$. At each iteration $t$, the current $\pi_t$ is more general than the candidate graph in the previous iteration, $\pi_{t-1}$. Whenever, $\pi_t$ cannot be generalized any further while ensuring $\pi_t \sqcup g_u \equiv g_d$, the algorithm stops, and returns the current graph $\pi_t$ (line 5). At each iteration, the algorithm computes the set $N$ of generalizations of $\pi_t$ that still ensure recovering $g_d$ when being unified with $g_d$. This is done without having to compute any unification operation, based on the following idea:
\begin{itemize}
\item We know that any generalization $g$ of $\pi_t$ must subsume $g_d$, since it is being generated using $\gamma$, starting from $\pi_0 = g_d$.
\item We know that $g_u \sqsubseteq g_d$ (otherwise, the remainder operation is not defined).
\item To check that $g \sqcup g_u \equiv g_d$ we just thus need to check that $\nexists g' \in G$ such that $g' \sqsubset g_d$, $g \sqsubseteq g'$, and $g_u \sqsubseteq g'$. To check this, we precompute the set $A$ that contains all the generalizations of $g_d$ that are subsumed by $g_u$, and then we make sure that $g$ does not subsume any of them. If these conditions are satisfied, and assuming $\gamma$ is complete, by Definition \ref{def:unification} $g_d$ must be a unifier of $g$ and $g_u$. 
\end{itemize}

Notice that step 8 of the algorithm is stochastic, since there might be multiple graphs that satisfy the definition of remainder. For the purposes of similarity assessment, it is enough to obtain one of them. A mode detailed discussion of this, in the context of refinement graphs can be found in the previous work of Onta\~{n}\'{o}n and Plaza \cite{ontanon2012similarity}.

Given the remainder operation, {\em disintegration} is defined as follows: 

\begin{definition}[Disintegration]
Given a finite refinement path $p = [g_0 = g_\bot, ..., g_n]$, a disintegration of the graph $g_n$ is the set $D(p) = \{r(g_{i+1}, g_i)|0 \leq i < n\}$.
\end{definition}

Notice that we use the expression ``a disintegration'' instead of ``the disintegration'', since there might be more than one disintegration because the remainder between two consecutive graphs in the path is not unique. In practice, given a graph $g$ and complete upward refinement operator $\gamma$, we can generalize it step by step by successive application of $\gamma$ (selecting one of the possible generalizations $\gamma$ produces stochastically), and use Algorithm \ref{alg:remainder} with $\gamma$ to generate a property at each step. We will write $D_\gamma(g)$ to a disintegration generated in this way.

Notice, moreover, that one of the unifiers (since the unification operation might not be unique) of all the properties in a disintegration of a graph $g$ is actually $g$ itself (i.e., we can ``reintegrate'' all the properties, to recover the original graph). This means that disintegration can preserves most of the information in the original graph (it preserves it all only when there is only a single unifier of all the properties). More theoretical properties of the disintegration operation are discussed by Onta\~{n}\'{o}n and Plaza \cite{ontanon2015refinement}.

\subsubsection{Properties-based Similarity Definition}

\begin{definition}[Properties-based Similarity]
Given two graphs $g_1$ and $g_2$, a complete upward refinement operator $\gamma$, and a subsumption relation $\sqsubseteq$, the {\em properties-based similarity measure}, $S_\pi$ is defined as follows:
\[S_\pi(g_1, g_2) = \frac{|\{\pi \in P \,\,\, | \,\,\, \pi \sqsubseteq g_1 \,\,\, \wedge \,\,\, \pi \sqsubseteq g_2\}|}
						 {|P|}\]
\noindent where $P = D_\gamma(g_1) \cup D_\gamma(g_2)$.
\end{definition}

In other words, $S_\pi(g_1, g_2)$ is defined as the number of properties that are shared between both graphs divided by the number of properties that at least one of them have. Moreover, in certain situations, it can be shown that $S_\pi(g_1, g_2)$ and $S_\lambda(g_1, g_2)$ are equivalent \cite{ontanon2012similarity}. The intuition behind this is that the number of properties they share should be equivalent to the length of the refinement path from $g_\top$ to their anti-unification.

\subsubsection{Property Weighting}

The main advantage of $S_\pi$ with respect to $S_\lambda$ is that it allows for weighting the contribution of each property in the similarity computation, and thus, the similarity measure can be fitted to a given supervised learning task. A procedure to compute these weights is as follows. Given a supervised machine learning task where training examples are of the form $(g, y)$, where $g$ is a graph and $y$ is a label, and given a training set $T = \{(g_1, y_1), ..., (g_n, y_n)\}$, and a complete upward refinement operator $\gamma$, we can assign a weight to a given property $\pi$ using the Quinlan's {\em information gain} measure \cite{Quinlan86ID3}:

\[w(\pi) = \frac{H(T_\pi) \times |T_\pi| + H(T \setminus T_\pi) \times |T \setminus T_\pi|}
			    {|T|} \]

\noindent where $T_\pi$ is the set of training examples from $T$ whose graph is subsumed by $\pi$, and $H(T)$ represents the entropy of the set of training examples $T$ with respect to the partition induced by their labels $y$.

Given such set of weights, we can now define the {\em weighted properties-based similarity} as follows:

\begin{definition}[Weighted Properties-based Similarity]
Given two graphs $g_1$ and $g_2$, a training set $T = \{(g'_1, y_1), ..., (g'_n, y_n)\}$, a complete upward refinement operator $\gamma$, and a subsumption relation $\sqsubseteq$, the {\em properties-based similarity measure}, $S_\pi$ is defined as follows:
\[S_{w\pi}(g_1, g_2) = \frac{\sum_{\pi \in P \,\,\, | \,\,\, \pi \sqsubseteq g_1 \,\,\, \wedge \,\,\, \pi \subseteq g_2} w(\pi)}
						 	{\sum_{\pi \in P} w(\pi)}\]
\noindent where $P = D_\gamma(g_1) \cup D_\gamma(g_2)$.
\end{definition}

Intuitively, $S_{w\pi}$ is equivalent to $S_{\pi}$, except that $S_{w\pi}$ counts the sum of the weights of the properties, whereas $S_{\pi}$ counts the number of properties. Also, notice that in practice, we can precompute the weights for all the properties resulting from disintegrating all the graphs in the training set, and we would only need to compute weights during similarity assessment, if the disintegration of either $g_1$ or $g_2$ yields a property that no other graph in the training set had.

\section{Related Work}

Two lines of work are related to $\rho$G, namely refinement operators, and similarity measures for structured representations. We summarize existing work on both areas here.

\subsection{Refinement Operators}

Since the introduction of refinement operators to ILP by Shapiro \cite{shapiro1981inductive}, there has been work on defining operators for multiple formalisms, and on understanding their fundamental properties. Here we provide a summary of existing refinement operators and their properties:
\begin{itemize}
	\item {\bf (Horn) Clauses}: 
	\begin{itemize}
		\item Shapiro \cite{shapiro1981inductive} defined a downward operator (proved to be non-complete by Laird \cite{laird2012learning}).
		\item Laird \cite{laird2012learning} proposed a complete version of Shapiro's downward operator. 
		\item Ling and Dawes \cite{ling1990sim} proposed an upward operator (not complete nor proper).
		\item Van der Laag and Nienhuys-Cheng \cite{van1994existence} proposed a complete version of Ling and Dawes'.
	\end{itemize}
	
	\item {\bf Datalog}: Esposito et al. \cite{esposito1996refinement} proposed the only known refinement operator for Datalog, which is complete, finite and proper for clauses sorted by $\theta$-subsumption under object identity.
	
	\item {\bf Description Logics}: 
	\begin{itemize}
		\item Badea and Nienhuys-Cheng \cite{badea2000refinement} defined a complete and proper, although redundant, operator for the $\mathcal{ALER}$ logic. 
		\item Lehmann and associates have defined refinement operators for $\mathcal{ALC}$ \cite{LehmannH07} (complete but not finite) and $\mathcal{EL}$ \cite{LehmannH09} (complete, proper and finite, used in our previous work \cite{DBLP:conf/iccbr/Sanchez-Ruiz-GranadosOGP11}). They proved that ideal refinement operators for more expressive Description Logics than $\mathcal{ALC}$ do not exist. 
		\item In our recent work \cite{sanchez2013refinement} we have shown an alternative route: and propose a refinement operator for DL conjunctive queries, instead of for concept definitions.
	\end{itemize}
	
	\item {\bf Feature Logics} (a.k.a., feature terms, or feature structures): in our previous work \cite{ontanon2012similarity}, we defined the only known upwards and downwards refinement operators for different subsets of Feature Logics (the downwards operators are complete, proper and finite, whereas the upwards one are proper and finite, but not complete).
	
	\item {\bf Partial Plans}: in our recent work \cite{sanchez2014least} we defined a refinement operator for partial plans (for plan recognition purposes). Its properties have not been theoretically proven.
	
	\item {\bf SQL}: Popescul and Ungar \cite{Popesculfeature2007} defined an operator for SQL queries, which is complete.
	
	\item {\bf Labeled Graphs}: to the best of our knowledge, the operators presented in this document, are the only known refinement operators for labeled graphs.
\end{itemize}

\subsection{Similarity Measures for Structured Representations}

% intro and similarity among basic data types:

The most common distance measures for numerical values are the different instantiations of the Minkowski distance, which when applied to scalars corresponds to computing their difference, and when applied to vectors, it generalizes the Manhattan distance  and to the Euclidean distance.

Concerning sets, the most well known measures are Tverski's \cite{tversky77similarity}, the Jaccard index (a special case of Tverski's), or the S\"orensen Index \cite{sorenson1948method}, with Jaccard being the most common: 
\[ J(A,B)=  \frac{|A \cup B|}{|A \cap B|}\] 

Similarity measures between sequences, such as the Levenshtein or Edit distance \cite{levenshtein1966binary}, have also been adapted to structured representations such as trees \cite{bille2005survey}. The edit distance equates the distance between two sequences to the cost of the edit operations that have to be done to one in order to obtain the second. A second common way to assess similarity between sequences is by representing them using a stochastic model (e.g., a Markov chain), and then comparing probability distribution that define the models, via measures such as the Kullback-Leibler divergence \cite{kullback1951information}. Other specific sequence similarity measures (such as Dynamic Time Warping \cite{keogh2005exact}) exist but are out of the scope of this section, since they have not been used for defining similarity measures for structured data. 

Finally, for some of these four basic ideas (Minkowski, Jaccard, Levenshtein, Kullback-Leibler), work exists on approaches that can weigh the different attributes/features/dimensions of the data in order to better fit a particular application domain, such as the Mahalanobis \cite{de2000mahalanobis} or Lin \cite{lin1998information} measures (this is often called ``metric learning'' in the literature \cite{xing2003distance}). Let us now see how these basic concepts have been used to define measures for structured representations.

% Horn clauses:

Work on similarity measures for Horn Clauses has been mainly carried out in Inductive Logic Programming (ILP). Hutchinson \cite{hutchinson1097metrics} presented a distance based on the least general generalization ({\em lgg}) of two clauses, (i.e., a clause that subsumes both),  The Hutchinson distance is computed as the addition of the sizes of the variable substitutions required transforming the {\em lgg} into each of the two clauses, 
which is analogous to the Jaccard index (with the {\em lgg} playing the role of the intersection, and the size of the variable substitutions as a measure of the difference in size between the intersection and the union). As pointed out in the literature \cite{ramon2002clustering,ontanon2012similarity}, this fails to take into account a lot of information. Another influential similarity measures for Horn Clauses is that in RIBL (Relational Instance-Based Learning) \cite{werner96relational}. RIBL's measure follows a ``hierarchical aggregation'' approach (also known as the ``local-global'' principle \cite{gabel2003learning}): the similarity of two objects is a function of the similarity of the values of their attributes (repeating this recursively). In addition to being only applicable to Horn Clauses, RIBL implicitly assumes that values ``further'' away from the root of an object will play a lesser role in similarity. Also, this hierarchical procedure makes RIBL not appropriate for objects that contain circularities. Finally, similarity measures have to be defined for different types of data inside of RIBL (e.g., for numerical values, categorical values, etc.). This last point is illustrated in the work of Horv\'ath et. al \cite{Horvath2001RIBL2}, an extension of RIBL able to deal with lists by incorporating an edit-distance whenever a list data type is reached. 
Other Horn Clauses similarity measures include the work of Bisson \cite{bisson1992similarity}, Nienhuys-Cheng \cite{nienhuys1997distance}, and of Ramon \cite{ramon2002clustering}.

% DL:

Concerning Description Logics (DL), Gonz\'alez-Calero et al. \cite{pedro99applyingdls} present a similarity measure for DL which, like RIBL, has problems with circularities in the data, and thus they preprocess the instances to remove such circularities. More recently, Fanizzi et al. \cite{Fanizzi2007semidistances} presented a similarity measure based on the idea of a ``committee of concepts''. They consider each concept in a given ontology to be a feature for each individual (belonging or not to that concept). The ratio of concepts that two individuals share corresponds to their similarity. This idea has been further developed by d'Amato \cite{d2007similarity}. SHAUD, presented by Armengol and Plaza~\cite{eva2003shaud}, is a similarity measure also following the ``hierarchical aggregation'' approach but designed for the feature logics (a.k.a., Feature Terms or Typed Feature Structures). SHAUD also assumes that the terms do not have circularities, and in the same way as RIBL it can handle numerical values by using specialized similarity measures. 
Bergmann and Stahl~\cite{Bergmann98} present a similarity measure for object-oriented representations based on the concepts of intra-class similarity and inter-class, defined in a recursive way, also following ``hierarchical aggregation'', making it more appropriate for tree representations. These similarity measures are an attempt to generalize standard Euclidean or Manhattan distances to structured data.

% Kernels:

Another related area is that of kernels for structured representations, which allow the application of techniques such as Support Vector Machines to structured data. Typically, kernels for graphs are based on the idea of finding common substructures between two graphs. For example, Kashima et al. \cite{Kashima2003kernel} present a kernel for graphs based on random walks. Fanizzi et al. \cite{fanizzi2008learning} also studied how to encapsulate their similarity measure for Description Logics into a kernel. For a survey on kernels for structured data the reader is referred to \cite{gartner2002kernels}. 

% Propositionalization:

Propositionalization (transforming structured instances into propositional) \cite{kramer2000propositionalization} has been used to apply standard similarity measures to structured data. For example, the measure introduced by Fanizzi et al. \cite{fanizzi2008dl} can be seen as such. Another related area is that of Formal Concept Analysis, where work on similarity assessment is starting to be studied \cite{formica2006ontology,alqadah2011similarity}. 

% bio:

Similarity measures for specific application domains, such as molecular structures in domains of biology or chemistry have also been studied \cite{Willett98chemical}, and they are typically grouped into three classes \cite{Raymond2003comparison}: sequence similarities (e.g., for DNA fragments), fingerprint-based (transform each molecule into a sequence of binary features representing whether a particular molecule exhibits particular properties or contains certain substructures) and graph-based (based on maximum common sub-graphs). The latter is a computationally expensive process, and thus there are a number of strategies to simplify the computations (e.g., \cite{Raymond2002rascal}). 

% Refinement operators:

The work presented in this paper builds upon recent work on similarity measures based on {\em refinement operators}. The key idea of these measures is to define similarity by only assuming the existence of a refinement operator for the target representation formalism. In this way, by just defining refinement operators for different representation formalisms, the same similarity measure can be used for all of these formalisms. Similarity measures for {\em feature terms} \cite{ontanon2012similarity}, Description Logics \cite{sanchez2012ontology,sanchez2013refinement}, and partial-order plans \cite{sanchez2014least} have been defined in this framework.

% Approximations:
%Finally, there has been work on numerical approximations to specific similarity measures. For example, approximations to the edit distance for strings are known \cite{AndoniOnak2009}. The work most related to this project is that of numerical approximations to similarity between combinatorial and geometrical data \cite{Matousek2002}. The intractability of such problems involving datasets whose inner structure obeys metric properties have increased interest in their approximate solutions with performance guarantees \cite{Gupta1999}. Embedding methods are among the most popular techniques for such approximation algorithms. Here the objective is to solve a possibly deformed version of a combinatorial problem in a ``simpler'' space, which has lower dimensionality or simpler structure \cite{Indyk2001}, making the approximation easier. Trees have become de facto metric structures for embedding problems due to the fact that approximating the solution for many NP-hard problems in general metrics can be done in polynomial time once data is embedded into tree metrics. Examples include closest and furthest pair, some clustering problems \cite{Indyk2001}, buy-at-bulk network design problem \cite{AwerbuchAzar1997}, group Steiner problem \cite{Garg1998}, and the $0$-extension problem \cite{Calinescu2001}.

% Summary:
In summary, there has been a significant amount of work on similarity assessment for structured representations, but the work has been carried out independently for different representation formalisms. For example, existing similarity measures defined for Horn Clauses are not applicable to labeled graphs. Second, the majority of similarity measures are based on principles, such as hierarchical aggregation, that introduce implicit biases that might not suit many application domains. The main goal behind the work that led to $\rho$G was to extend the general framework of similarity measures based on refinement operators to {\em directed labeled graphs} by introducing appropriate refinement operators, showing that the ideas generalize to a wide variety of representation formalisms.

%------------------------------------------------------------------------
%------------------------------------------------------------------------
%------------------------------------------------------------------------

\section{Conclusions}

This document has presented the foundations behind the $\rho$G (RHOG) library for directed labeled graphs\footnote{\url{https://github.com/santiontanon/RHOG}}. Specifically, the library offers the following functionalities: graph subsumption, unification, anti-unification, refinement and similarity assessment. Most of those functionalities are supported by a collection of refinement operators, of which we have shown the theoretical properties in this document.

$\rho$G builds upon our previous work on defining similarity measures for structured representations, with the goal of providing a foundation for structured machine learning algorithms that are independent of the representation formalism being used. Although $\rho$G focuses on directed labeled graphs, all the algorithms implemented in $\rho$G are applicable to any other representation formalism for which refinement operators and subsumption relations can be defined.

All the operations currently provided by $\rho$G are calculated systematically (i.e., $\rho$G does not yet provide numerical approximations to any of the operations, such as similarity assessment). As such, even if most of the operations can be performed in fairly large graphs, computational cost is expected to be high for very large graphs. As part of our future work, we plan to provide such numerical approximations, which will allow the application of refinement-operator-based algorithms to large-scale applications.

\bibliographystyle{plain}
\bibliography{references}






%\end{thebibliography}

\newpage

\appendix

\section{Appendix: Proofs}\label{appendix:proofs}

\subsection{Preliminary Definitions and Results}

Let us start by introducing some necessary definitions and preliminary proofs that will facilitate the proofs of the main propositions of this document.

\begin{definition}[Cover]
Given two graphs $g_1 = \langle V_1, E_1, l_1 \angle$ and $g_2 = \langle V_2, E_2, l_2 \angle$, such that $g_1 \sqsubseteq g_2$. We define the {\em cover} of $g_1$ over $g_2$, as the subgraph of $g_2$ containing only those vertices and edges that are referred to by the subsumption mapping $m$. More formally, $\mathcal{C}^{g_1 \sqsubseteq g_2} = \langle \mathcal{C}_v, \mathcal{C}_e, l_2 \rangle$, where:
\begin{itemize}
\item In the case of regular subsumption ($\sqsubseteq$) or subsumption relative to $\preceq$ ($\sqsubseteq_{\preceq}$):
\[\mathcal{C}_v = \{w \in V_2 | \exists v \in V_1 : m(v) = w \}\]
\[\mathcal{C}_w = \{e_2 \in E_2 | \exists e_1 \in E_1 : m(e_1) = e_2 \} \]

\item In the case of trans-subsumption ($\overrightarrow{\sqsubseteq}$) or trans-subsumption relative to $\preceq$ ($\overrightarrow{\sqsubseteq}_{\preceq}$):
\[\mathcal{C}_v = \{w \in V_2 | \exists v \in V_1 : m(v) = w \} \cup \{w \in V_2 | \exists e \in E_1 : \exists (w_1,w_2) \in m_e(e) : w = w_1 \,\, \vee \,\, w = w_2\}\]
\[\mathcal{C}_e = \bigcup_{e \in E_1} m_e(e)\]

%\[\mathcal{C}_w = \{(w_1,w_2) \in E_2 | \exists e \in E_1 : m_e = [u_1, ..., u_k] \,\, \wedge \,\, w_1 = u_i \,\, w_2 = u_{i+1} \,\, 1 \leq i < k \} \]		

\end{itemize}
Moreover, when there cannot be any confusion, we will drop the superindexes, and just note the cover as $\mathcal{C}$. Also, some times we will only be interested in $\mathcal{C}_v$, which we will call the {\em vertex cover}, or in $\mathcal{C}_e$, which we will call the {\em edge cover}.
\end{definition}

\begin{definition}[Delta]
Given two graphs $g_1$ and $g_2$, such that $g_1 \sqsubseteq g_2$. We define the {\em delta} of $g_1$ over $g_2$ as the set of vertices and edges that do not belong to the cover:
\begin{itemize}
\item The vertex delta is defined as: $\Delta_v = V_2 \setminus \mathcal{C}_v$.
\item The edge delta is defined as: $\Delta_e = E_2 \setminus \mathcal{C}_e$.
\item The delta is defined as the graph $\Delta^{g_1 \sqsubseteq g_2} = \langle \Delta_v, \Delta_e, l_2 \rangle$ (notice that $\Delta$ is not necessarily connected).
\end{itemize} 
\end{definition}

Moreover, in the same way as with the cover, when there cannot be any confusion, we will drop the superindexes, and just note the delta as $\Delta$. Figure \ref{fig:cover-delta} shows an illustration of the intuition behind the {\em cover} and {\em delta} concepts.

\begin{figure}[tb]
    \centering
    \includegraphics[width=0.6\columnwidth]{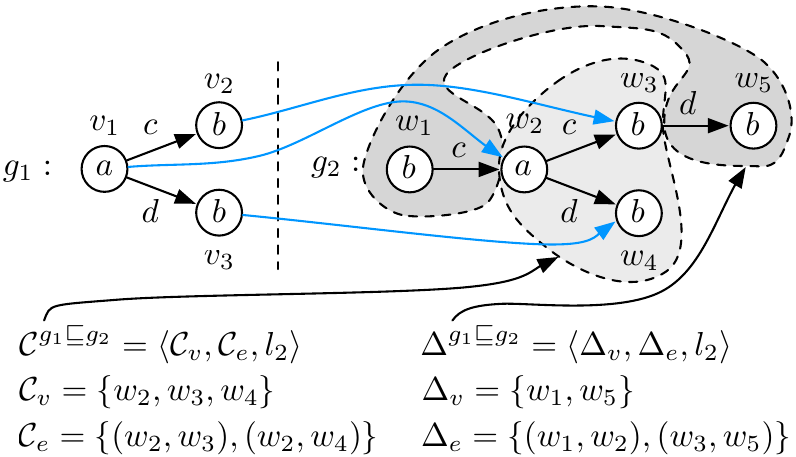}
    \caption{{\bf Illustration of the concept of {\em cover} and {\em delta}}: two graphs $g_1 = \langle V_1, E_1, l_1 \rangle$, and $g_2 = \langle V_2, E_2, l_2 \rangle$, such that $g_1 \sqsubseteq g_2$ via regular subsumption. The blue arrows represent the subsumption mapping $m$. As the figure shows, the cover basically contains all the vertices and edges referred to by the subsumption mapping, and the delta contains the rest of vertices and edges.}
    \label{fig:cover-delta}
\end{figure}

\begin{proposition}\label{prop:cover-connected}
Given two graphs $g_1 = \langle V_1, E_1, l_1 \rangle$ and $g_2 = \langle V_2, E_2, l_2 \rangle$, such that $g_1 \sqsubseteq g_2$, if $g_1$ is connected, then $\mathcal{C}^{g_1 \sqsubseteq g_2}$ is also connected.
\end{proposition}
\begin{proof}
Imagine that $\mathcal{C}^{g_1 \sqsubseteq g_2}$ is not connected. This happens if and only if there are two vertices $w_1, w_2 \in \mathcal{C}_v$ that are disconnected (i.e., we cannot find a path between them). For the sake of a graph being connected or not, we do not care about the direction of the edges here. If two such vertices exist, then let $v_1, v_2 \in V_1$ be two vertices from $g_1$ such that $m(v_1) = w_1$, and $m(v_2) = w_2$ (notice that two such vertices always exist by the definition of {\em cover}). Since $g_1$ is connected, there is a path consisting of the following vertices $[u_1, ..., u_k]$, such that $u_1 = v_1$ and $u_k = v_2$. We can now construct a sequence $[m(u_1), ..., m(u_k)]$, and show that since in $g_1$ there is an edge $e_i$ connecting each $u_i$ to $u_{i+1}$, then the corresponding edge must also exist in $g_2$ by the definition of subsumption. This is a contradiction, since we assumed such sequence of vertices didn't exist, and therefore, we can conclude that $\mathcal{C}^{g_1 \sqsubseteq g_2}$ must be connected.
\end{proof}

\begin{proposition}\label{prop:delta-leaf-or-nonbridge}
Given two connected graphs $g_1 = \langle V_1, E_1, l_1 \rangle$ and $g_2 = \langle V_2, E_2, l_2 \rangle$, such that $g_1 \sqsubseteq g_2$, if $\Delta_v \neq \emptyset$, then one of the two following conditions must be satisfied:
\begin{itemize}
\item $\exists v \in \Delta_v$ such that $v$ is a leaf.
\item $\exists e \in \Delta_e$ such that $e$ is not a bridge.
\end{itemize}
This results will help us proof the completeness of the upward refinement operators.
\end{proposition}
\begin{proof}
Let us proof this by contradiction. Imagine that none of the two conditions is satisfied. Since $\Delta_v \neq \emptyset$, then only two situations might arise:
\begin{itemize}
\item If $V_1 = \emptyset$, then $\Delta = g_2$. Then, if we ignore the directionality of the edges in $g_2$: (a) if $g_2$ has no loops, then, by definition, it is a tree, and thus, must have leaves (contradiction), (b) else, if $g_2$ has loops, then none of the edges that constitute any of the loops is a bridge (contradiction).
\item If $V_1 \neq \emptyset$, then $\mathcal{C}_v \neq \emptyset$. Since $g_2$ is connected, there must be an edge $e = (w_1, w_2) \in E_2$ such that $w_1 \in \mathcal{C}_v$ and $w_2 \in \Delta_v$. Since $w_2 \in \Delta_v$, we know that $e \in \Delta_e$, and thus, by assumption, $e$ is not a bridge. By the definition of bridge, $e$ splits the vertices of $g_2$ into two disjoint subsets $V_2 = V_2^a \cup V_2^b$, where $w_1 \in V_2^a$ and $w_2 \in V_2^b$. Since $w_1 \in \mathcal{C}_v$ and we know (by Proposition \ref{prop:cover-connected}) that $\mathcal{C}$ is connected, then $\mathcal{C}_v \subseteq V_2^a$. Thus, since $V_2^a$ and $V_2^b$ are disjoint, we know that $V_2^b \subseteq \Delta_v$. Now, if we just look at the subgraph $g$ of $g_2$ formed by the vertices in $V_2^b$ (which is a subgraph of $\Delta$), there can only be two situations: 
	\begin{itemize}
	\item If $g$ has no loops (recall we do not care about the directionality of the edges), then, by definition, $g$ is a tree. If $g$ only has one vertex ($w_2$), then this vertex is a leaf (contradiction). If it has more than one vertex, then, given that any vertex in an undirected tree can be considered the root, if we consider $w_2$ to be the root, then there must be at least one leaf (contradiction).
	\item If $g$ has loops, then, none of the edges that constitute any of the loops is a bridge, contradiction.
	\end{itemize}
\end{itemize}
\end{proof}

%------------------------------------------------------------------------

\subsection{Proofs for Downward Refinement of FDLG ($\rho_f$)}

\noindent {\bf Proposition \ref{prop:rho-f-finite-complete}}. {\em
The downward refinement operator $\rho_f$ defined by the rewrite rules R0, R1, R2, and R3 above is locally finite, and complete for the quasi-ordered set $\langle G, \sqsubseteq \rangle$ (where $\sqsubseteq$ represents regular subsumption).}

\begin{proof}
Let us proof each of the properties separately:
\begin{itemize}
\item {\em $\rho_f$ is locally finite}: the number of refinements generated by each rewrite rule corresponds to the number of possible values that the variables in the {\em applicability conditions} (left-hand side of the rule) can take. Thus, let us consider the different rewrite rules:
	\begin{itemize}
	\item R0: $v_*$ represents a ``new'' vertex, and thus can only take one value, and $a$ can take $|L|$ values. Since $|L|$ is a finite number, R0 can only generate a finite number of refinements.
	\item R1: $v_*$ represents a ``new'' vertex, and thus can only take one value. $v_1$ can take $n = |V|$ values, $a$ can take $|L|$ values and $b$ can take $|L|$ values. Thus, the total number of possible value bindings for R1 is $1 \times n \times |L|^2$. Since $n$ and $|L|$ are finite numbers, R1 can only generate a finite number of refinements.
	\item R2: This is analogous to R1.
	\item R3: $v_1$ and $v_2$ can take at most $|V|$ values each, and $a$ can take $|L|$ values. Thus, R2 can generate at most $n^2|L|$ refinements, which is a finite number.
	\end{itemize}
\item {\em $\rho_f$ is complete}: consider any two DLGs $g_u = \langle V_u, E_u, l_u \rangle , g_d = \langle V_d, E_d, l_d \rangle \in G$, such that $g_u \sqsubseteq g_d$. We need to proof that we can get to $g_d$ from $g_u$ by repeated application of the refinement operator. Since $g_u \sqsubseteq g_d$, we know there is a mapping $m$ that satisfies Definition \ref{def:subsumption}. We will distinguish two cases:
	\begin{itemize}
	\item If $g_d \sqsubseteq g_u$: then there is no need to apply the refinement operator, and we are done.
	\item If $g_d \not\sqsubseteq g_u$: in this case, we can get to $g_d$ using the following procedure. At each step $t$ of the procedure, we will construct a new graph $g_t = \langle V_t, E_t, l_t \rangle$, via downward refinement of $g_{t-1}$ getting one step closer to $g_d$. Each graph $g_t$ subsumes $g_d$ via a mapping $m_t$. In the first step $t = 0$, $g_0 = g_u$, then:
		\begin{itemize}
    		\item If $|V_t| = 0$: then let $v_2 \in V_d$ be any of the vertices in $g_d$, then R0 (with $a = l_2(v_2)$) can be used to generate a refinement $g_{t+1}$ that clearly subsumes $g_d$.		
    		\item If the vertex cover of $g_t$ over $g_d$ does not include all the vertices in $g_d$, i.e.,  $|\mathcal{C}_v| < |V_d|$: then let $v_2, w_2 \in V_d$ be any two vertices in $g_d$ such that $v_d \not\in \mathcal{C}_v$, $w_2 \in \mathcal{C}_v$, and either $(v_2, w_2) \in E_d$ or $(w_2, v_2) \in E_d$ (notice that two such vertices must exist, since $g_d$ is a connected graph). Then:
				\begin{itemize}
				\item if $(v_2, w_2) \in E_d$ then R1 (with $v_1 \in V_t$ s.t. $m_t(v_1) = w_2$, $a = l_d(v_2)$, and $b = l_d((v_2, w_2))$) can be used to generate a refinement $g_{t+1}$ that subsumes $g_d$, by extending the mapping $m_{t}$, with $m_{t+1}(v_1) = v_2$. 
				\item Alternatively, if $(w_2, v_2) \in E_d$ then R2 (with $v_1 \in V_t$ s.t. $m_t(v_1) = w_2$, $a = l_d(v_2)$, and $b = l_d((w_2, v_2))$) can be used to generate $g_{t+1}$, which also subsumes $g_d$, by extending the mapping $m_t$, with $m_{t+1}(v_1) = v_2$.
				\end{itemize}
    		\item Otherwise, if the edge cover does not include all the edges in $g_d$, i.e., $|\mathcal{C}_e| < |E_d|$: in this case, given there must be some $(w_1, w_2) \in E_d$ such that $(w_1, w_2) \not\in \mathcal{C}_e$, and R3 (with $v_1 \in V_t$ s.t. $m_t(v_1) = v_2$, $v_2 \in V_t$ s.t. $m_t(v_2) = w_2$) can be used to generate a refinement $g_{t+1}$ that subsumes $g_d$ via the same mapping $m_{t+1} = m_t$ (since we are only adding an edge, and we are adding it so that by using $m_t$, the subsumption conditions are still satisfied). 
    		\item Otherwise, we can show that $g_d \sqsubseteq g_t$ via a mapping $m_{t}^{-1}$ constructed in the following way: $\forall_{v_2 \in V_d} m_t^{-1}(v_2) = v_1 : m_t(v_1) = v_2$. In other words, we just have to invert the mapping. Notice that if we are not enforcing object identity, more than one vertex in $V_t$ might map to the same vertex $v_2 \in V_d$, when constructing $m_t^{-1}$ we just need to pick any of those vertices in $V_t$ as the mapping of $v_2$. It is trivial to see that this mapping satisfies the subsumption conditions in Definition \ref{def:subsumption}.
		\end{itemize}		
	Notice that the procedure always terminates in a finite number of steps because at each step we are always increasing either $|\mathcal{C}_v|$  (with R0, R1 or R2), or $|\mathcal{C}_e|$ (with R3), and never decreasing either of them. Since $|\mathcal{C}_v|$ and $|\mathcal{C}_e|$ are upper-bounded by $|V_d|$ and $|E_d|$ respectively, the previous process terminates in a finite number of steps.
	\end{itemize}
\end{itemize}
\end{proof}

%------------------------------------------------------------------------

\noindent {\bf Proposition \ref{prop:rho-f-ideal}}. {\em
The downward refinement operator $\rho_f$ defined by the rewrite rules R0, R1, R2, and R3 above is ideal (locally finite, complete, and proper) for the quasi-ordered set $\langle G, \sqsubseteq \rangle$ (where $\sqsubseteq$ represents regular subsumption), when we impose the Object Identity constraint.
}
\begin{proof}
By Proposition \ref{prop:rho-f-finite-complete} $\rho_f$ is already locally finite and complete, so, we just need to prove that under object identity, $\rho_f$ is also proper. Notice that given two graphs $g_u = \langle V_u, E_u, l_u \rangle$ and $g_d = \langle V_d, E_d, l_d \rangle$, such that $g_d \sqsubseteq g_u$ object identity in regular subsumption, implies the following: since $v_1 \neq v_2 \implies m(v_1) \neq m(v_2)$, we know that $|V_u| = |\{m(v)|v \in V_u\}|$. Thus, this implies that for two $g_u$ and $g_d$ to subsume each other (i.e., for being equivalents), we must have that $|V_u| = |V_d|$, which then implies that $|E_u| = |E_d|$. Since R0, R1, R2 and R3, all increase either the number of vertices or the number of edges of a graph, under object identity, a graph $g$ can never be equivalent to any refinement of $g$ generated by R0, R1, R2 or R3. Thus, $\rho_f$ is proper. 
\end{proof}

%------------------------------------------------------------------------

\subsection{Proofs for Downward Refinement of FDLG using Trans-Subsumption ($\rho_{tf}$)}

\noindent {\bf Proposition \ref{prop:rho-tf-finite-complete}}. {\em
The downward refinement operator $\rho_{tf}$ defined by the rewrite rules R0, R1, R2, R3, and R4 above is locally finite, and complete for the quasi-ordered set $\langle G, \overrightarrow{\sqsubseteq} \rangle$ (where $\overrightarrow{\sqsubseteq}$ represents trans-subsumption).
}
\begin{proof}
Let us proof each of the properties separately:
\begin{itemize}
\item {\em $\rho_{tf}$ is locally finite}: the number of refinements generated by each rewrite rule corresponds to the number of possible values that the variables in the {\em applicability conditions} (left-hand side of the rule) can take. Thus, let us consider the different rewrite rules:
	\begin{itemize}
	\item by Proposition \ref{prop:rho-f-finite-complete}, R0, R1, R2 and R3 produce only a finite number of refinements.
	\item R4: $v_*$ represents a ``new'' vertex, and thus can only take one value. $(v_1, v_2)$ can only be bound in $|E|$ different ways, $a$ can take $|L|$ values, and $b$ is determined by the binding of $(v_1, v_2)$. Thus, the maximum number of possible value bindings for R4 is $1 \times |E| \times |L| \times 1$, and thus can only generate a finite number of refinements.
	\end{itemize}

\item {\em $\rho_{tf}$ is complete}: consider any two DLGs $g_u = \langle V_u, E_u, l_u \rangle , g_d = \langle V_d, E_d, l_d \rangle \in G$, such that $g_u \overrightarrow{\sqsubseteq} g_d$. We need to proof that we can get to $g_d$ from $g_u$ by repeated application of the refinement operator. Since $g_u \overrightarrow{\sqsubseteq} g_d$, we know there are two mappings $m$, $m_e$ that satisfy Definition \ref{def:trans-subsumption}. We will distinguish two cases:
	\begin{itemize}
	\item If $g_d \overrightarrow{\sqsubseteq} g_u$: then there is no need to apply the refinement operator, and we are done.
	\item If $g_d \overrightarrow{\not\sqsubseteq} g_u$: in this case, analogously to what we did in Proposition \ref{prop:rho-f-finite-complete}, we will define a procedure that ensures getting to $g_d$ by repeated application of $\rho_{tf}$ starting from $g_u$ in a finite number of steps. At each step $t$ of the procedure, we will construct a new graph $g_t = \langle V_t, E_t, l_t \rangle$ via downward refinement of $g_{t-1}$ which subsumes $g_d$ via the vertex mapping $m_t$ and the edge mapping $m_{e,t}$, and gets one step closer to $g_d$. In the first step, $t = 0$, $g_0 = g_u$, then:
		\begin{itemize}
    		\item If $|V_t| = 0$: then let $v_2 \in V_d$ be any of the vertices in $g_d$, then R0 (with $a = l_d(v_2)$) can be used to generate a refinement $g_{t+1}$ that clearly subsumes $g_d$.
			\item Otherwise, if $\exists e \in E_t: m_{e,t}(v) = [w_1, ..., w_k]$ such that $k>2$, then this means that there is some edge in $g_t$ that has been mapped to more than one edge in $g_d$. In this case, we can use R4 to split $e$ by using the following bindings: $(v_1, v_2) = e$, $a = l_d(w_2)$. The resulting graph still subsumes $g_d$ by expanding the mapping $m_t$ with $m_{t+1}(v_*) = w_2$, and by updating $m_{e,t}$ as follows: $m_{e,t+1}((v_1,v_*)) = [w_1,w_2]$, and $m_{e,t+1}(v_*,v_2) = [w_2, ..., w_k]$.
    		\item If the vertex cover does not include all the vertices in $g_d$, i.e.,  $|\mathcal{C}_v| < |V_d|$: then let $v_2, w_2 \in V_d$ be any two vertices in $g_d$ such that $v_2 \not\in \mathcal{C}_v$, $w_2 \in \mathcal{C}_v$, and either $(v_2, w_2) \in E_d$ or $(w_2, v_2) \in E_d$ (notice that two such vertices must exist, since $g_d$ is a connected graph, and also, that $v_2$ cannot be part of any of the paths to which edges in $g_t$ are mapped via $m_{e,t}$ since if we have reached this point in the procedure, it means that every single edge in $g_t$ is mapped to a single edge in $g_d$). Then:
				\begin{itemize}
				\item if $(v_2, w_2) \in E_d$ then R1 (with $v_1 \in V_t$ s.t. $m_{t}(v_1) = w_2$, $a = l_d(v_2)$, and $b = l_d((v_2, w_2))$) can be used to generate a refinement $g_{t+1}$ that subsumes $g_d$, by extending the mappings, with $m_{t+1}(v_1) = v_2$, and $m_{e,t+1}((v_1, w_1)) = [v_2,w_2]$ (where $w_1 \in V_t$ s.t. $m_t(w_1) = w_2$). 
				\item Alternatively, if $(w_2, v_2) \in E_d$ then R2 (with $v_1 \in V_t$ s.t. $m_t(v_1) = w_2$, $a = l_d(v_2)$, and $b = l_d((w_2, v_2))$) can be used to generate $g_{t+1}$, which also subsumes $g_d$, by extending the mapping $m_t$, with $m_{t+1}(v_1) = v_2$, and $m_{e,t+1}((w_1,v_1)) = [w_2,v_2]$ (where $w_1 \in V_t$ s.t. $m_t(w_1) = w_2$).
				\end{itemize}
    		\item Otherwise, if the edge cover does not include all the edges in $g_d$, i.e., $|\mathcal{C}_e| < |E_d|$: in this case, there must be some $(w_1, w_2) \in E_d$ such that $(w_1, w_2) \not\in \mathcal{C}_e$, and R3 (with $v_1 \in V_t$ s.t. $m_t(v_1) = v_2$, $v_2 \in V_1$ s.t. $m_t(v_2) = w_2$) can be used to generate a refinement $g_{t+1}$ that subsumes $g_d$ via the same mapping $m_t$ (since we are only adding an edge, and we are adding it so that by using $m_{t+1} = m_t$, the subsumption conditions are still satisfied by expanding $m_{e,t+1}((v_1,v_2)) = [w_1,w_2]$). 
    		\item Otherwise, we can show that $g_d \sqsubseteq g_t$ via a mapping $m^{-1}_t$ constructed in the following way: $\forall_{v_2 \in V_d} m^{-1}_t(v_2) = v_1 : m_t(v_1) = v_2$. In other words, we just have to invert the mapping. $m_{e,t}$ can also be inverted directly, since at this point in the procedure each edge $e = (v_1,v_2) \in V_t$ is mapped to a sequence of just two vertices $m_{e,t}(e) = [w_1,w_2]$, and thus, we can invert the mapping as $m_{e,t}^{-1}((w_1,w_2)) = [v_1,v_2]$. Moreover, notice that if we are not enforcing object identity, more than one vertex in $V_t$ might map to the same vertex in $v_2 \in V_d$, when constructing $m^{-1}_t$ we just need to pick any of those vertices in $V_t$ as the mapping of $v_2$. The same thing might happen with $m_{e,t}$. It is trivial to see that this mapping satisfies the subsumption conditions in Definition \ref{def:trans-subsumption}.
		\end{itemize}	
	Notice that the previous procedure always terminates in a finite number of steps because at each step we are always increasing either $|\mathcal{C}_v|$  (with R0, R1 or R2), or $|\mathcal{C}_e|$ (with R3), and never decreasing either of them, or decreasing the size of one of the paths that $m_{e,t}$ maps to by splitting these paths in half (with R4) until we reach the minimum size of two, and never increasing the length of any of these paths. Since $|\mathcal{C}_v|$ and $|\mathcal{C}_e|$ are upper-bounded by $|V_d|$ and $|E_d|$ respectively, and the length of the paths that $m_{e,t}$ maps to is lower-bounded by 2, the previous process terminates in a finite number of steps.
	\end{itemize}
\end{itemize}
\end{proof}

%------------------------------------------------------------------------

\noindent {\bf Proposition \ref{prop:rho-tf-ideal}}. {\em
The downward refinement operator $\rho_{tf}$ defined by the rewrite rules R0, R1, R2, R3, and R4 above is ideal (locally finite, complete, and proper) for the quasi-ordered set $\langle G, \overrightarrow{\sqsubseteq} \rangle$ (where $\overrightarrow{\sqsubseteq}$ represents trans-subsumption), when we impose the Object Identity constraint.
}
\begin{proof}
By Proposition \ref{prop:rho-tf-finite-complete} $\rho_{tf}$ is already locally finite and complete, so, we just need to prove that under object identity, $\rho_{tf}$ is also proper. 

Notice that given two graphs $g_u = \langle V_u, E_u, l_u \rangle$ and $g_d = \langle V_d, E_d, l_d \rangle$, such that $g_d \sqsubseteq g_u$ object identity in trans-subsumption implies that $v_1 \neq v_2 \implies m(v_1) \neq m(v_2)$. Thus $|V_u| = |\{m(v)|v \in V_u\}| \leq |V_d|$. Thus, if $g_u$ is to subsume $g_d$, then $|V_u| \leq |V_d|$. Analogously, if $g_d$ is to subsume $g_u$, then $|V_u| \geq |V_d|$. Thus, for $g_u$ and $g_d$ to be equivalent, we must have that $|V_u| = |V_d|$. This further implies that both the $m_e$ mapping through which $g_u$ subsumes $g_d$ can map each edge to just a sequence of two vertices (since each vertex of $g_u$ is already mapped to a different vertex in $g_d$, and there is no vertex in $g_d$ such that no vertex of $g_u$ is mapped to it). Thus, since each edge of $g_u$ is mapped to exactly just one edge in $g_d$ (and vice versa) we know also that  $|E_u| = |E_d|$. Since R0, R1, R2, R3, and R4 all increase either the number of vertices or the number of edges of a graph, under object identity, a graph $g$ can never be equivalent to any refinement of $g$ generated by R0, R1, R2, R3, or R4. Thus, $\rho_{tf}$ is proper
\end{proof}

%------------------------------------------------------------------------

\subsection{Proofs for Upward Refinement of FDLG ($\gamma_f$)}

\noindent {\bf Proposition \ref{prop:rho-upward-f-finite-complete}}. {\em
The upward refinement operator $\gamma_f$ defined by the rewrite rules UR0 and UR1 above is locally finite, and complete for the quasi-ordered set $\langle G, \sqsubseteq \rangle$ (where $\sqsubseteq$ represents regular subsumption).
}
\begin{proof}
Let us proof each of the properties separately:
\begin{itemize}
\item {\em $\gamma_f$ is locally finite}: the number of refinements generated by each rewrite rule corresponds to the number of possible values that the variables in the {\em applicability conditions} (left-hand side of the rule) can take. Thus, let us consider the different rewrite rules:
	\begin{itemize}
	\item UR0: $e$ can only be instantiated to each of the $|E|$ different edges of the graph, which is a finite number, and thus UR0 can only generate a finite number of refinements.	
	\item UR1: $v$ can only be instantiated to each of the $|V|$ vertexes of the graph, and $v$ determines $E_v$. Thus, UR1 can generate at most $|V|$ refinements, which is a finite number.
	\end{itemize}
\item {\em $\gamma_f$ is complete}: consider any two DLGs $g_u = \langle V_u, E_u, l_u \rangle , g_d = \langle V_d, E_d, l_d \rangle \in G$, such that $g_u \sqsubseteq g_d$. We need to proof that we can get to $g_u$ from $g_g$ by repeated application of the refinement operator. Since $g_u \sqsubseteq g_d$, we know there is a mapping $m$, that satisfies Definition \ref{def:subsumption}. We will distinguish two cases:
	\begin{itemize}
	\item If $g_d \sqsubseteq g_u$: then there is no need to apply the refinement operator, and we are done.
	\item If $g_d \not\sqsubseteq g_u$: in this case, analogously to what we did in Proposition \ref{prop:rho-f-finite-complete}, we will define a procedure that ensures getting to $g_u$ by repeated application of $\gamma_{t}$ starting from $g_d$ in a finite number of steps. At each step $t$ of the procedure, we will construct a new graph $g_t = \langle V_t, E_t, l_t \rangle$ which is subsumed by $g_u$ via mapping $m_t$ using upward refinement of $g_{t-1}$, and getting one step closer to $g_u$. In the first step, $t=0$, $g_0 = g_d$, then:
		\begin{itemize}
    		\item If the vertex cover of $g_u$ over $g_t$ is not the complete $E_t$, i.e., if $|\mathcal{C}_v| < |V_t|$ and thus $\Delta_v \neq \emptyset$, then by Proposition \ref{prop:delta-leaf-or-nonbridge}, one of these two situations must arise:
			\begin{itemize}
%				\item If $V_1 = \emptyset$ and $|V_2| = 1$. Then, UR0 can be used to generate a new graph by removing the only vertex in $V_2$.
				\item If $\exists w \in \Delta_v$ such that $w$ is a {\em leaf} (i.e., it is only connected to the rest of $g_t$ via a single edge) then, we can use UR0 to generate a new graph $g_{t+1}$ still subsumed by $g_u$ by removing $w$.
				\item Otherwise, there must be a non-bridge edge $e = (w_1, w_2) \in E_t$, such that either $w_1 \in \Delta_v$ or $w_2 \in \Delta_v$. UR1 can be used to generate a new graph $g_{t+1}$ still subsumed by $g_u$ by removing $e$.
			\end{itemize}
			\item Otherwise, it must be the case that $|\mathcal{C}_v| = |V_t|$. In this case, one of the two conditions must be satisfied:
			\begin{itemize}
				\item If there is a pair of vertices $v_1, v_2 \in V_u$ such that $(v_1, v_2) \not\in E_u$, but $(m(v_1), m(v_2)) \in E_t$, then, we know for sure that $(m(v_1), m(v_2))$ is not a bridge (otherwise, $g_u$ would not be connected). Thus, we can use UR1 to generate a new graph $g_{t+1}$ still subsumed by $g_u$ by removing such edge.
				\item Otherwise, we can show that $g_t \sqsubseteq g_u$ via a mapping $m^{-1}_t$ constructed in the following way: $\forall_{v_2 \in V_t} m^{-1}_t(v_2) = v_1 : m_t(v_1) = v_2$. In other words, we just have to invert the mapping. Moreover, notice that if we are not enforcing object identity, more than one vertex in $V_u$ might map to the same vertex in $v_2 \in V_t$, when constructing $m^{-1}_t$ we just need to pick any of those vertices in $V_u$ as the mapping of $v_2$.
			\end{itemize}	
			\item Notice that the previous procedure always terminates in a finite number of steps because at each step we are always decreasing either the number of edges or the number of vertices of the graph, and never increasing any of them. Moreover, since the number of edges and vertices are trivially lower-bounded by 0, the previous process must terminate in a finite number of steps.			
		\end{itemize}
	\end{itemize}
\end{itemize}
\end{proof}

%------------
------------------------------------------------------------

\noindent {\bf Proposition \ref{prop:rho-upward-f-ideal}}. {\em
The upward refinement operator $\gamma_f$ defined by the rewrite rules UR0 and UR1 above is ideal (locally finite, complete and proper) for the quasi-ordered set $\langle G, \sqsubseteq \rangle$ (where $\sqsubseteq$ represents regular subsumption), when we impose the Object Identity constraint.
}
\begin{proof}
By Proposition \ref{prop:rho-upward-f-finite-complete} $\gamma_f$ is already locally finite and complete, so, we just need to prove that under object identity, $\gamma_f$ is also proper. As noted above in the proof of Proposition \ref{prop:rho-f-ideal}, given two graphs $g_u = \langle V_u, E_u, l_u \rangle$ and $g_d = \langle V_d, E_d, l_d \rangle$, such that $g_d \sqsubseteq g_u$ object identity in regular subsumption implies the following: since $v_1 \neq v_2 \implies m(v_1) \neq m(v_2)$, this means that $|V_u| = |\{m(v)|v \in V_u\}|$. Thus, this implies that for $g_u$ and $g_d$ to subsume each other (i.e., for being equivalents), we must have that $|V_u| = |V_d|$, which then implies that $|E_u| = |E_d|$. Since UR0 and UR1 both decrease the number of edges or vertices of a graph, under object identity, a graph $g$ can never be equivalent to any upward refinement of $g$ generated by UR0 or UR1. Thus, $\gamma_f$ is proper. 
\end{proof}

%------------------------------------------------------------------------

\subsection{Proofs for Upward Refinement of FDLG using Trans-Subsumption ($\gamma_{tf}$)}

\noindent {\bf Proposition \ref{prop:rho-upward-tf-finite-complete}}. {\em
The upward refinement operator $\gamma_{tf}$ defined by the rewrite rules UR0, UR1 and UR2 above is locally finite, and complete for the quasi-ordered set $\langle G, \overrightarrow{\sqsubseteq} \rangle$ (where $\overrightarrow{\sqsubseteq}$ represents trans-subsumption).
}

\begin{proof}
Let us proof each of the properties separately:
\begin{itemize}
\item {\em $\gamma_{tf}$ is locally finite}: the number of refinements generated by each rewrite rule corresponds to the number of possible values that the variables in the {\em applicability conditions} (left-hand side of the rule) can take. Thus, let us consider the different rewrite rules:
	\begin{itemize}
	\item by Proposition \ref{prop:rho-upward-f-finite-complete}, UR0 and UR1 produce only a finite number of refinements.
	\item UR2: the only two free variables are $e_1$ and $e_2$, which can take at most $|E|$ possible values. Thus, UR2 can at most generate $|E|^2$ refinements, which is a finite number.
	\end{itemize}

\item {\em $\gamma_{tf}$ is complete}: consider any two DLGs $g_u = \langle V_u, E_u, l_u \rangle , g_d = \langle V_d, E_d, l_d \rangle \in G$, such that $g_u \overrightarrow{\sqsubseteq} g_d$. We need to proof that we can get to $g_u$ from $g_d$ by repeated application of the refinement operator. Since $g_u \overrightarrow{\sqsubseteq} g_d$, we know there are two mappings $m$, $m_e$ that satisfy Definition \ref{def:trans-subsumption}. We will distinguish two cases:
	\begin{itemize}
	\item If $g_d \overrightarrow{\sqsubseteq} g_u$: then there is no need to apply the refinement operator, and we are done.
	\item If $g_d \overrightarrow{\not\sqsubseteq} g_u$: in this case, analogously to what we did to proof completeness in previous propositions, we will define a procedure that ensures getting to $g_u$ by repeated application of $\gamma_{tf}$ starting from $g_d$ in a finite number of steps. At each step $t$ of the procedure, we will construct a new graph $g_t = \langle V_t, E_t, l_t \rangle$ using upward refinement of $g_{t-1}$, and getting one step closer to $g_u$. In the first step $t=0$ and $g_0 = g_d$.
		\begin{itemize}
		\item If $\Delta_v \neq \emptyset$, then that means that are vertices in $V_t$ to which no vertex or  edge in $g_u$ is mapped. By Proposition \ref{prop:delta-leaf-or-nonbridge}, one of these two situations must arise:
			\begin{itemize}
%			\item If $V_u = \emptyset$ and $|V_t| = 1$. Then, UR0 can be used to generate a new graph by removing the only vertex in $V_t$.
			\item If $\exists w \in \Delta_v$ such that $w$ is a {\em leaf} (i.e., it is only connected to the rest of $g_t$ via a single edge), then, we can use UR0 to generate a new graph by removing $w$.
			\item Otherwise, there must be a non-bridge edge $e = (w_1, w_2) \in E_t$, such that either $w_1 \in \Delta_v$ or $w_2 \in \Delta_v$. UR1 can be used to generate a new graph by removing $e$.
			\end{itemize}
		\item Otherwise, we know that $\Delta_v = \emptyset$. Now, if $\Delta_e \neq \emptyset$, then we know that none of the edges in $\Delta_e$ can be a bridge (since $g_u$ is a connected graph), and thus, we can use UR0 to generate refinements that remove those edges one at a time.
		\item Otherwise, we know that both $\Delta_v = \emptyset$ and $\Delta_e = \emptyset$. Now only two situations can arise:
			\begin{itemize}
			\item If there is any edge $e \in E_u$ such that $m_e(e) = [w_1, ..., w_k]$, where $k>2$, then we know that no vertices in $g_u$ are mapped to any of the vertices $w_2, ...,w_{k-1}$. Since $\Delta_e = \emptyset$, the only edges that $w_i$ ($2 \leq i \leq k-1$) can have thus are those linking with $w_{i-1}$ and $w_{i+1}$. Under this circumstances, we can use UR2 to generate a refinement removing a vertex $w_i$, where $2 \leq i \leq k-1$.
			\item Otherwise, we can show that $g_t \overrightarrow{\sqsubseteq} g_u$ via a mapping $m^{-1}$ constructed in the following way: $\forall_{v_2 \in V_t} m^{-1}(v_2) = v_1 : m(v_1) = v_2$. In other words, we just have to invert the mapping. $m_e$ can also be inverted directly, since at this point in the procedure each edge $e = (v_1,v_2) \in V_u$ is mapped to a sequence of just two vertices $m_e(e) = [w_1,w_2]$, and thus, we can invert the mapping as $m_e^{-1}((w_1,w_2)) = [v_1,v_2]$. Moreover, notice that if we are not enforcing object identity, more than one vertex in $V_u$ might map to the same vertex in $v_2 \in V_t$, when constructing $m^{-1}$ we just need to pick any of those vertices in $V_u$ as the mapping of $v_2$. The same thing might happen with $m_e$. It is trivial to see that this mapping satisfies the subsumption conditions in Definition \ref{def:trans-subsumption}.
			\end{itemize}
		\end{itemize} 
	Notice that the previous procedure always terminates in a finite number of steps because at each step we are always decreasing either the number of edges or the number of vertices of the graph, and never increasing any of them. Moreover, since the number of edges and vertices are trivially lower-bounded by 0, the previous process must terminate in a finite number of steps.
	\end{itemize}
\end{itemize}
\end{proof}

%------------------------------------------------------------------------

\noindent {\bf Proposition \ref{prop:rho-upward-tf-ideal}}. {\em
The upward refinement operator $\gamma_{tf}$ defined by the rewrite rules UR0, UR1 and UR2 above is ideal (locally finite, complete and proper) for the quasi-ordered set $\langle G, \overrightarrow{\sqsubseteq} \rangle$ (where $\overrightarrow{\sqsubseteq}$ represents trans-subsumption), when we impose the Object Identity constraint.
}
\begin{proof}
By Proposition \ref{prop:rho-upward-tf-finite-complete} $\gamma_{tf}$ is already locally finite and complete, so, we just need to prove that under object identity, $\gamma_{tf}$ is also proper. As shown in the proof of Proposition \ref{prop:rho-upward-f-ideal}, for two graphs $g_1$ and $g_2$ to be equivalent under trans-subsumption, they must have the same exact number of vertices and edges. Since UR0, UR1, and UR2 all reduce either the number of vertices or edges, upward refinements of a graph $g$ generated by UR0, UR1, and UR2 can never be equivalent to $g$, and thus $\gamma_{tf}$ is proper.
\end{proof}

%------------------------------------------------------------------------

\subsection{Proofs for Downward Refinement of ODLG ($\rho_\preceq$)}

\noindent {\bf Proposition \ref{prop:rho-prec-finite-complete}}. {\em
The downward refinement operator $\rho_{\preceq}$ defined by the rewrite rules R0PO, R1PO, R2PO, R3PO, R4PO, and R5PO above is locally finite and complete for the quasi-ordered set $\langle G, \sqsubseteq_{\prec} \rangle$ (where $\sqsubseteq_{\prec}$ represents subsumption relative to the partial order $\prec$).
}
\begin{proof}
Let us proof each of the properties separately:
\begin{itemize}
\item {\em $\rho_{\preceq}$ is locally finite}: the number of refinements generated by each rewrite rule corresponds to the number of possible values that the variables in the {\em applicability conditions} (left-hand side of the rule) can take. Thus, let us consider the different rewrite rules:
	\begin{itemize}
	\item R0PO: $v_*$ represents a ``new'' vertex, and thus can only take one value, so this operator can only generate a single refinement.
	\item R1PO: $v_*$ represents a ``new'' vertex, and thus can only take one value, $v_1$ can take $|V|$ values, and thus this operator can only generate $|V|$ refinements, which is a finite number.
	\item R2PO: this is analogous to R1PO.
	\item R3PO: $v_1$ and $v_2$ can only take $|V|$ values each, and thus, this operator can only generate $|V|^2$ refinements, which is a finite number.
	\item R4PO: $v_1$ can only take $|V|$ values, and $b$ can at most take $|L|$ values. Thus this operator can only generate $|V|\times|L|$ refinements, which is a finite number.
	\item R5PO: $e$ can only take $|E|$ values, and $b$ can at most take $|L|$ values. Thus this operator can only generate $|E|\times|L|$ refinements, which is a finite number.
	\end{itemize}
	
\item {\em $\rho_{\preceq}$ is complete}: 
consider any two DLGs $g_u = \langle V_u, E_u, l_u \rangle , g_d = \langle V_d, E_d, l_d \rangle \in G$, such that $g_u \sqsubseteq_{\preceq} g_d$. We need to proof that we can get to $g_d$ from $g_u$ by repeated application of the refinement operator. Since $g_u \sqsubseteq_{\preceq} g_d$, we know there is a mapping $m$ that satisfies Definition \ref{def:subsumption-po}. We will distinguish two cases:
	\begin{itemize}
	\item If $g_d \sqsubseteq_{\preceq} g_u$: then there is no need to apply the refinement operator, and we are done.
	\item If $g_d \not\sqsubseteq_{\preceq} g_u$: in this case, we can get to $g_d$ using the following procedure. At each step $t$ of the procedure, we will construct a new graph $g_t = \langle V_t, E_t, l_t \rangle$, which is a downward refinement of $g_{t-1}$ and gets one step closer to $g_d$. In the first step $t=0$ and $g_0 = g_u$:
		\begin{itemize}
    		\item If $|V_t| = 0$: then then R0PO can be used to generate a refinement $g_t$ that clearly still subsumes $g_d$ (since $g_d$ must have at least one vertex, otherwise it would subsume $g_t$).
			\item If there is a vertex $v \in V_t$ such that $l_t(v_1) \neq l_d(m(v_1))$, then we can use R4PO to specialize the label of $v_1$, by selecting a label $b$ that satisfies the applicability conditions of R4PO, while at the same time satisfying $b \preceq l_d(m(v_1))$ (notice that such label $b$ must exist, since otherwise, $g_t$ would not subsume $g_d$).			
    		\item If the vertex cover does not include all the vertices in $g_d$, i.e.,  $|\mathcal{C}_v| < |V_d|$: then let $v_2, w_2 \in V_d$ be any two vertices in $g_d$ such that $v_2 \not\in \mathcal{C}_v$, $w_2 \in \mathcal{C}_v$, and either $(v_2, w_2) \in E_d$ or $(w_2, v_2) \in E_d$ (notice that two such vertices must exist, since $g_d$ is a connected graph). Then:
				\begin{itemize}
				\item If $(v_2, w_2) \in E_d$ then R1PO (with $v_1 \in V_t$ s.t. $m(v_1) = w_2$) can be used to generate a refinement $g_t$ that subsumes $g_d$, by extending the mapping $m$, with $m(v_*) = v_2$. 
				\item Alternatively, if $(w_2, v_2) \in E_d$ then R2PO (with $v_1 \in V_t$ s.t. $m(v_1) = w_2$) can be used to generate $g_t$, which also subsumes $g_d$, by extending the mapping $m$, with $m(v_*) = v_2$.
				\end{itemize}
			\item If there is an edge $e \in E_t$ such that $l_t(e) \neq l_d(m(e))$, then we can use R5PO to specialize the label of $e$, by selecting a label $b$ that satisfies the applicability conditions of R5PO, while at the same time satisfying $b \preceq l_d(m(e))$ (notice that such label $b$ must exist, since otherwise, $g_t$ would not subsume $g_d$).
    		\item Otherwise, if the edge cover does not include all the edges in $g_d$, i.e., $|\mathcal{C}_e| < |E_d|$: in this case, there must be some $(w_1, w_2) \in E_d$ such that $(w_1, w_2) \not\in \mathcal{C}_e$, and R3PO (with $v_1 \in V_t$ s.t. $m(v_1) = v_2$, $v_2 \in V_t$ s.t. $m(v_2) = w_2$) can be used to generate a refinement $g_t$ that subsumes $g_d$ via the same mapping $m$ (since we are only adding an edge, and we are adding it so that by using $m$, the subsumption conditions are still satisfied). Notice that $v_1$ and $v_2$ must exist, since reached this point, we have that $|\mathcal{C}_v| = |V_d|$.
    		\item Otherwise, we can show that $g_d \sqsubseteq_{\preceq} g_t$ via a mapping $m^{-1}$ constructed in the following way: $\forall_{v_2 \in V_d} m^{-1}(v_2) = v_1 : m(v_1) = v_2$. In other words, we just have to invert the mapping. Notice that if we are not enforcing object identity, more than one vertex in $V_t$ might map to the same vertex in $v_2 \in V_d$, when constructing $m^{-1}$ we just need to pick any of those vertices in $V_t$ as the mapping of $v_2$. It is trivial to see that this mapping satisfies the subsumption conditions in Definition \ref{def:subsumption-po}.
		\end{itemize}		
	Notice that the procedure always terminates in a finite number of steps because at each step we are always increasing either $|\mathcal{C}_v|$  (with R0PO, R1PO or R2PO), or $|\mathcal{C}_e|$ (with R3PO), and never decreasing either of them, or making a label more specific (with R4PO or R5PO), and never making it more general. Since $|\mathcal{C}_v|$ and $|\mathcal{C}_e|$ are upper-bounded by $|V_d|$ and $|E_d|$ respectively, and the number of times we can make a label more specific is bounded by the size of $|L|$, the previous process terminates in a finite number of steps.
	\end{itemize}
\end{itemize}
\end{proof}

%------------------------------------------------------------------------

\noindent {\bf Proposition \ref{prop:rho-prec-ideal}}. {\em
The downward refinement operator  $\rho_{\preceq}$ defined by the rewrite rules R0PO, R1PO, R2PO, R3PO, R4PO and R5PO above is ideal (locally finite, complete, and proper) for the quasi-ordered set $\langle G, \sqsubseteq_{\prec} \rangle$ (where $\sqsubseteq_{\prec}$ represents subsumption relative to the partial order $\prec$), when we impose the Object Identity constraint.
}
\begin{proof}
By Proposition \ref{prop:rho-prec-finite-complete} $\rho_{\preceq}$ is already locally finite and complete, so, we just need to prove that under object identity, $\rho_{\preceq}$ is also proper. 
Notice that object identity in subsumption relative to the partial order $\prec$ implies $|V_u| = |\{m(v)|v \in V_u\}|$ (since $v_1 \neq v_2 \implies m(v_1) \neq m(v_2)$). Thus, this implies that for two graphs $g_u = \langle V_u, E_u, l_u \rangle , g_d = \langle V_d, E_d, l_d \rangle \in G$ to subsume each other (i.e., for being equivalents), we must have that $|V_u| = |V_d|$, which then implies that $|E_u| = |E_d|$. 
Since R0PO, R1PO, R2PO and R3PO all increase either the number of vertices or the number of edges of a graph, under object identity, a graph $g$ can never be equivalent to any refinement generated by R0PO, R1PO, R2PO and R3PO. 

Let us now consider R4PO. R4PO takes a graph $g_u$ and generates another one, $g_d$, that is identical except for the label of a vertex $v$, which is changed from a label $a$ to a label $b$, such that $b$ is more specific than $a$: $a \preceq b$ and $a \neq b$. We can split the set of vertices $V_u$ into two subsets: $V_u^b = \{v \in V_u | b \preceq l_u(v) \}$ (those vertices with a label equal or more specific than $b$) and $V_u^{\overline{b}} = V_u \setminus V_u^b$. We can do the same with the vertices of $V_d$: $V_d^b = \{v \in V_d | b \preceq l_d(v) \}$ and $V_d^{\overline{b}} = V_d \setminus V_d^b$. By construction, we also know that $V_d^b = V_u^b \cup \{v\}$ (since the only change to the graph is the label of vertex $v$), and thus $|V_d^{b}| = |V_u^b| + 1$. Now, if $g_d$ is to subsume $g_u$, a mapping $m^{-1}$ from $V_d$ to $V_u$ must exist that satisfies the subsumption conditions. Given Definition \ref{def:subsumption-po}, vertices in $V_d^{b}$ must be mapped via $m^{-1}$ to different vertices in $V_u^{b}$ (due to the Object Identity constraint), however, since $|V_u^{b}| < |V_d^b|$, this mapping cannot exist, and thus $g_d$ cannot subsume $g_u$ if $g_d$ is generated with R4PO.

Finally, let us consider R5PO. Notice that object identity also implies that when $g_u$ and $g_d$ subsume each other, we have that $|E_u| = |E_d|$. This is because, the subsumption mapping $m$ between edges also defines a mapping between edges, where $m((v,w)) = (m(v), m(w))$, given that $v \neq w$ implies $m(v) \neq m(w)$, we also know that if $e_1 \neq e_2$ (where $e_1, e_2 \in E_u$), then $m(e_1) \neq m(e_2)$. Thanks to this fact, R5PO cannot generate a refinement that subsumes $g_u$. The proof is analogous to the one for R4PO in the previous paragraph, but this time splitting the set of edges $E_u$ into two subsets, instead of the set of vertices.

Since none of the refinements generated by R0PO, R1PO, R2PO, R3PO, R4PO or R5PO of a graph $g_u$ can subsume $g_u$, we have that $\rho_\preceq$ is proper.
\end{proof}

%------------------------------------------------------------------------

\subsection{Proofs for Downward Refinement of ODLG using Trans-Subsumption ($\rho_{t\preceq}$)}

\noindent {\bf Proposition \ref{prop:rho-tprec-finite-complete}}. {\em
The downward refinement operator $\rho_{t\preceq}$ defined by the rewrite rules R0PO, R1PO, R2PO, R3PO, R4PO, R5PO, and R6PO above is locally finite, and complete for the quasi-ordered set $\langle G, \overrightarrow{\sqsubseteq}_{\prec} \rangle$ (where $\overrightarrow{\sqsubseteq}_{\prec}$ represents trans-subsumption relative to the partial order $\prec$).
}
\begin{proof}
Let us proof each of the properties separately:
\begin{itemize}
\item {\em $\rho_{t\preceq}$ is locally finite}: the number of refinements generated by each rewrite rule corresponds to the number of possible values that the variables in the {\em applicability conditions} (left-hand side of the rule) can take. Thus, let us consider the different rewrite rules:
	\begin{itemize}
	\item by Proposition \ref{prop:rho-prec-finite-complete}, R0PO, RPO, R2PO, R3PO, R4PO and R5PO produce only a finite number of refinements.
	\item R6PO: $v_*$ represents a ``new'' vertex, and thus can only take one value, $(v_1, v_2)$ can only be bound in $|E|$ different ways, and $b$ is determined by the binding of $(v_1, v_2)$. Thus, the maximum number of possible value bindings for R6PO is $1 \times |E| \times 1$, and thus can only generate a finite number of refinements.
	\end{itemize}

\item {\em $\rho_{t\preceq}$ is complete}: consider any two DLGs $g_u = \langle V_u, E_u, l_u \rangle , g_d = \langle V_d, E_d, l_d \rangle \in G$, such that $g_u \overrightarrow{\sqsubseteq}_{\prec} g_d$. We need to proof that we can get to $g_d$ from $g_u$ by repeated application of the refinement operator. Since $g_u \overrightarrow{\sqsubseteq}_{\prec} g_d$, we know there are two mappings $m$, $m_e$ that satisfy Definition \ref{def:trans-subsumption-po}. We will distinguish two cases:
	\begin{itemize}
	\item If $g_d \overrightarrow{\sqsubseteq}_{\prec} g_u$: then there is no need to apply the refinement operator, and we are done.
	\item If $g_d \overrightarrow{\not\sqsubseteq}_{\prec} g_u$: in this case, analogously to what we did in previous propositions, we will define a procedure that ensures getting to $g_d$ by repeated application of $\rho_{t\prec}$ starting from $g_u$ in a finite number of steps. At each step $t$ of the procedure, we will construct a new graph $g_t = \langle V_t, E_t, l_t \rangle$, which is a downward refinement of $g_{t-1}$ and gets one step closer to $g_d$, in the first step $t=0$ and $g_0 = g_u$:
		\begin{itemize}
    		\item If $|V_t| = 0$: then then R0PO can be used to generate a refinement $g_t$ that clearly still subsumes $g_d$ (since $g_d$ must have at least one vertex, otherwise it would subsume $g_t$).
			\item Otherwise, if $\exists e \in E_t: m_e(v) = [w_1, ..., w_k]$ such that $k>2$, then this means that there is some edge in $g_t$ that has been mapped to more than one edge in $g_d$. In this case, we can use R6PO to split $e$ by using the following bindings: $(v_1, v_2) = e$. The resulting graph still subsumes $g_d$ by expanding the mapping $m$ with $m(v_*) = w_2$, and by updating $m_e$ as follows: $m_e((v_1,v_*)) = [w_1,w_2]$, and $m_e(v_*,v_2) = [w_2, ..., w_k]$.
			\item If there is a vertex $v \in V_t$ such that $l_t(v_1) \neq l_d(m(v_1))$, then we can use R4PO to specialize the label of $v_1$, by selecting a label $b$ that satisfies the applicability conditions of R4PO, while at the same time satisfying $b \preceq l_d(m(v_1))$ (notice that such label $b$ must exist, since otherwise, $g_t$ would not subsume $g_d$).			
    		\item If the vertex cover does not include all the vertices in $g_d$, i.e.,  $|\mathcal{C}_v| < |V_d|$: then let $v_2, w_2 \in V_d$ be any two vertices in $g_d$ such that $v_2 \not\in \mathcal{C}_v$, $w_2 \in \mathcal{C}_v$, and either $(v_2, w_2) \in E_d$ or $(w_2, v_2) \in E_d$ (notice that two such vertices must exist, since $g_d$ is a connected graph, and also, that $v_2$ cannot be part of any of the paths to which edges in $g_t$ are mapped via $m_e$ since if we have reached this point in the procedure, it means that every single edge in $g_t$ is mapped to a single edge in $g_d$). Then:
				\begin{itemize}
				\item If $(v_2, w_2) \in E_d$ then R1PO (with $v_1 \in V_t$ s.t. $m(v_1) = w_2$) can be used to generate a refinement $g_t$ that subsumes $g_d$, by extending the mapping $m$, with $m(v_*) = v_2$. 
				\item Alternatively, if $(w_2, v_2) \in E_d$ then R2PO (with $v_1 \in V_t$ s.t. $m(v_1) = w_2$) can be used to generate $g_t$, which also subsumes $g_d$, by extending the mapping $m$, with $m(v_*) = v_2$.
				\end{itemize}
			\item If there is an edge $e \in E_t$ such that $l_t(e) \neq l_d(m(e))$, then we can use R5PO to specialize the label of $e$, by selecting a label $b$ that satisfies the applicability conditions of R5PO, while at the same time satisfying $b \preceq l_d(m(e))$ (notice that such label $b$ must exist, since otherwise, $g_t$ would not subsume $g_d$).
    		\item Otherwise, if the edge cover does not include all the edges in $g_d$, i.e., $|\mathcal{C}_e| < |E_d|$: in this case, there must be some $(w_1, w_2) \in E_d$ such that $(w_1, w_2) \not\in \mathcal{C}_e$, and R3PO (with $v_1 \in V_t$ s.t. $m(v_1) = v_2$, $v_2 \in V_t$ s.t. $m(v_2) = w_2$) can be used to generate a refinement $g_t$ that subsumes $g_d$ via the same mapping $m$ (since we are only adding an edge, and we are adding it so that by using $m$, the subsumption conditions are still satisfied by expanding $m_e((v_1,v_2)) = [w_1,w_2]$). Notice that $v_1$ and $v_2$ must exist, since reached this point, we have that $|\mathcal{C}_v| = |V_d|$.
 
    		\item Otherwise, we can show that $g_d \overrightarrow{\not\sqsubseteq}_{\prec} g_t$ via a mapping $m^{-1}$ constructed in the following way: $\forall_{v_2 \in V_d} m^{-1}(v_2) = v_1 : m(v_1) = v_2$. In other words, we just have to invert the mapping. Notice that if we are not enforcing object identity, more than one vertex in $V_t$ might map to the same vertex in $v_2 \in V_d$, when constructing $m^{-1}$ we just need to pick any of those vertices in $V_t$ as the mapping of $v_2$. It is trivial to see that this mapping satisfies the subsumption conditions in Definition \ref{def:subsumption-po}.
    		\item Otherwise, we can show that $g_d \overrightarrow{\sqsubseteq}_{\prec} g_t$ via a mapping $m^{-1}$ constructed in the following way: $\forall_{v_2 \in V_d} m^{-1}(v_2) = v_1 : m(v_1) = v_2$. In other words, we just have to invert the mapping. $m_e$ can also be inverted directly, since at this point in the procedure each edge $e = (v_1,v_2) \in V_t$ is mapped to a sequence of just two vertices $m_e(e) = [w_1,w_2]$, and thus, we can invert the mapping as $m_e^{-1}((w_1,w_2)) = [v_1,v_2]$. Moreover, notice that if we are not enforcing object identity, more than one vertex in $V_t$ might map to the same vertex in $v_2 \in V_d$, when constructing $m^{-1}$ we just need to pick any of those vertices in $V_t$ as the mapping of $v_2$. The same thing might happen with $m_e$. It is trivial to see that this mapping satisfies the subsumption conditions in Definition \ref{def:trans-subsumption-po}.
		\end{itemize}	
	Notice that the procedure always terminates in a finite number of steps because at each step we are always increasing either $|\mathcal{C}_v|$  (with R0PO, R1PO or R2PO), or $|\mathcal{C}_e|$ (with R3PO), and never decreasing either of them, making a label more specific (with R4PO or R5PO), and never making it more general,
or decreasing the size of one of the paths that $m_e$ maps to by splitting these paths in half (with R6PO) until we reach the minimum size of two, and never increasing the length of any of these paths. Since $|\mathcal{C}_v|$ and $|\mathcal{C}_e|$ are upper-bounded by $|V_d|$ and $|E_d|$ respectively, and the number of times we can make a label more specific is bounded by the size of $|L|$, the previous process terminates in a finite number of steps.
	\end{itemize}
\end{itemize}
\end{proof}

%------------------------------------------------------------------------

\noindent {\bf Proposition \ref{prop:rho-tprec-ideal}}. {\em
The downward refinement operator  $\rho_{t\preceq}$ defined by the rewrite rules R0PO, R1PO, R2PO, R3PO, R4PO, R5PO, and R6PO above is ideal (locally finite, complete, and proper) for the quasi-ordered set $\langle G, \overrightarrow{\sqsubseteq}_{\prec} \rangle$ (where $\overrightarrow{\sqsubseteq}_{\prec}$ represents trans-subsumption relative to the partial order $\prec$), when we impose the Object Identity constraint.
}
\begin{proof}
By Proposition \ref{prop:rho-tprec-finite-complete} $\rho_{t\preceq}$ is already locally finite and complete, so, we just need to prove that under object identity, $\rho_{t\preceq}$ is also proper. Notice that object identity in trans-subsmption relative to a partial order $\prec$, implies that $v_1 \neq v_2 \implies m(v_1) \neq m(v_2)$. Thus $|V_u| = |\{m(v)|v \in V_u\}| \leq |V_d|$. Thus, if $g_u$ is to subsume $g_d$, then $|V_u| \leq |V_d|$. Analogously, if $g_d$ is to subsume $g_u$, then $|V_u| \geq |V_d|$. Thus, for $g_u$ and $g_d$ to be equivalent, we must have that $|V_u| = |V_d|$. This further implies that both the $m_e$ mapping through which $g_u$ subsumes $g_d$ can map each edge to just a sequence of two vertices (since each vertex of $g_u$ is already mapped to a different vertex in $g_d$, and there is no vertex in $g_d$ such that no vertex of $g_u$ is mapped to it). Thus, since each edge of $g_u$ is mapped to exactly just one edge in $g_d$ (and vice versa) we know also that  $|E_u| = |E_d|$. Since R0PO, R1PO, R2PO, R3PO, and R6PO all increase either the number of vertices or the number of edges of a graph, under object identity, a graph $g$ can never be equivalent to any refinement generated by R0PO, R1PO, R2PO, R3PO, and R6PO. 

Considering R4PO and R5PO, the proof of Proposition \ref{prop:rho-prec-ideal} for R4PO and R5PO applies here, and thus $\rho_{t\preceq}$ is proper
\end{proof}

%------------------------------------------------------------------------

\subsection{Proofs for Upward Refinement of ODLG ($\gamma_\preceq$)}

\noindent {\bf Proposition \ref{prop:gamma-prec-finite-complete}}. {\em
The upward refinement operator $\gamma_{\preceq}$ defined by the rewrite rules UR0PO, UR1PO, UR2PO, and UR3PO above is locally finite and complete for the quasi-ordered set $\langle G, \sqsubseteq_{\prec} \rangle$ (where $\sqsubseteq_{\prec}$ represents subsumption relative to the partial order $\prec$).
}
\begin{proof}
Let us proof each of the properties separately:
\begin{itemize}
\item {\em $\gamma_\preceq$ is locally finite}: the number of refinements generated by each rewrite rule corresponds to the number of possible values that the variables in the {\em applicability conditions} (left-hand side of the rule) can take. Thus, let us consider the different rewrite rules:
	\begin{itemize}
	\item UR0PO: $v_1$ can take $|V|$ values, $a$ is determined by $v_1$, and $b$ can take at most $|L|$ values. Thus, the upper bound on the number of refinements this rule can generate is $|V| \times |L|$, which is a finite number.
	\item UR1PO: $e$ can take $|E|$ values, $a$ is determined by $e$, and $b$ can take at most $|L|$ values. Thus, the upper bound on the number of refinements this rule can generate is $|E| \times |L|$, which is a finite number.
	\item UR2PO: the only free variable is $e$, which can take at most $|E|$ values. Thus this rule can also only generate a finite number of refinements.
	\item UR3PO: $v$ can take at most $|V|$ values (a finite amount), and $E_v$ is determined by $v$. Thus this rule can also only generate a finite number of refinements.
	\end{itemize}

\item {\em $\gamma_\preceq$ is complete}: consider any two DLGs $g_u = \langle V_u, E_u, l_u \rangle , g_d = \langle V_d, E_d, l_d \rangle \in G$, such that $g_u \sqsubseteq_\preceq g_d$. We need to proof that we can get to $g_d$ from $g_u$ by repeated application of the refinement operator. Since $g_u \sqsubseteq_\preceq g_d$, we know there is a mapping $m$, that satisfy Definition \ref{def:subsumption-po}. We will distinguish two cases:
	\begin{itemize}
	\item If $g_d \sqsubseteq_\preceq g_u$: then there is no need to apply the refinement operator, and we are done.
	\item If $g_d \not\sqsubseteq_\preceq g_u$: in this case, analogously to what we did in previous propositions, we will define a procedure that ensures getting to $g_u$ by repeated application of $\gamma_\preceq$ starting from $g_d$ in a finite number of steps. At each step $t$ of the procedure, we will construct a new graph $g_t = \langle V_t, E_t, l_t \rangle$ using upward refinement of $g_{t-1}$, getting one step closer to $g_u$. In the first step, $t=0$ and $g_0 = g_d$:
		\begin{itemize}
			\item if $\exists w \in \Delta_v$ such that $l_d(w) \neq \top$, then UR0PO can be used to generate a graph with the label of $w$ generalized one step closer to $\top$. Since $w$ does not belong to the cover, changing its label does not affect any of the subsumption conditions, and the resulting graph would still be subsumed by $g_u$.
			\item Otherwise, if $\exists v \in V_u$ such that $l_u(v) \neq l_t(m(v))$, then since $g_u \sqsubseteq_\preceq g_t$, we know that $l_u(v) \preceq l_t(m(v))$, and UR0PO can be used to generalize $l_t(m(v))$ on step closer to $l_u(v)$, by selecting $b$ such that $l_u(v) \preceq b$ in order to ensure that the subsumption relation between $g_u$ and the new refinement is still satisfied (notice that such $b$ must exist since $l_u(v) \preceq l_t(m(v))$).
			\item Otherwise, if $\exists e \in \Delta_e$ such that $l_t(e) \neq \top$, then UR1PO can be used to generate a graph with the label of $e$ generalized one step closer to $\top$. Since $e$ does not belong to the cover, changing its label does not affect any of the subsumption conditions, and the resulting graph would still be subsumed by $g_u$.
			\item Otherwise, if $\exists e \in E_u$ such that $l_u(e) \neq l_t(m(e))$, then since $g_u \sqsubseteq_\preceq g_t$, we know that $l_u(e) \preceq l_t(m(e))$, and UR1PO can be used to generalize $l_t(m(e))$ on step closer to $l_u(e)$, by selecting $b$ such that $l_u(e) \preceq b$ in order to ensure that the subsumption relation between $g_u$ and the new refinement is still satisfied (notice that such $b$ must exist since $l_u(e) \preceq l_t(m(e))$).		
    		\item At this point, the steps above ensure that: $\forall_{v \in V_u} l_u(v) = l_t(m(v))$, $\forall_{e \in E_u} l_u(e) = l_t(m(e))$, $\forall_{w \in \Delta_v} l_t(w) = \top$, and $\forall_{e \in \Delta_e} l_t(e) = \top$. Now, if $\Delta_v \neq \emptyset$, then by Proposition \ref{prop:delta-leaf-or-nonbridge}, one of these two situations must arise:
			\begin{itemize}
				\item If $\exists w \in \Delta_v$ such that $w$ is a {\em leaf} (i.e., it is only connected to the rest of $g_t$ via a single edge) then, we can use UR3PO to generate a new graph by removing $w$ (since we know that $l_t(w) = \top$).
				\item Otherwise, there must be a non-bridge edge $e = (w_1, w_2) \in E_t$, such that either $w_1 \in \Delta_v$ or $w_2 \in \Delta_v$. UR2PO can be used to generate a new graph by removing $e$ (since we know that $l_t(e) = \top$).
			\end{itemize}
			\item Otherwise, $\Delta_v = \emptyset$ and thus $|\mathcal{C}_v| = |V_t|$. In this case, one of the two conditions must be satisfied:
			\begin{itemize}
				\item If there is a pair of vertices $v_1, v_2 \in V_u$ such that $(v_1, v_2) \not\in E_u$, but $(m(v_1), m(v_2)) \in E_t$, then, we know for sure that $(m(v_1), m(v_2))$ is not a bridge (otherwise, $g_u$ would not be connected). Thus, we can use UR2PO to generate a graph by removing such edge (which we know has label $\top$).
				\item Otherwise, we can show that $g_t \sqsubseteq g_u$ via a mapping $m^{-1}$ constructed in the following way: $\forall_{v_2 \in V_t} m^{-1}(v_2) = v_1 : m(v_1) = v_2$. In other words, we just have to invert the mapping. Moreover, notice that if we are not enforcing object identity, more than one vertex in $V_u$ might map to the same vertex in $v_2 \in V_t$, when constructing $m^{-1}$ we just need to pick any of those vertices in $V_u$ as the mapping of $v_2$.
			\end{itemize}	
			\item Notice that the previous procedure always terminates in a finite number of steps because at each step we are always decreasing either the number of edges or the number of vertices of the graph (with UR2PO or UR3PO), and never increasing any of them, or making a label more general (with UR0PO or UR1PO) and never making it more specific. Moreover, since the number of edges and vertices are trivially lower-bounded by 0, and the number of times we can make a label more general is bounded by the size of $|L|$, the previous process must terminate in a finite number of steps.		
		\end{itemize}	
	\end{itemize}
\end{itemize}
\end{proof}

%------------------------------------------------------------------------

\noindent {\bf Proposition \ref{prop:gamma-prec-ideal}}. {\em
The upward refinement operator  $\gamma_{\preceq}$ defined by the rewrite rules UR0PO, UR1PO, UR2PO, and UR3PO above is ideal (locally finite, complete, and proper) for the quasi-ordered set $\langle G, \sqsubseteq_{\prec} \rangle$ (where $\sqsubseteq_{\prec}$ represents subsumption relative to the partial order $\prec$), when we impose the Object Identity constraint.
}
\begin{proof}
By Proposition \ref{prop:gamma-prec-finite-complete} $\gamma_\preceq$ is already locally finite and complete, so, we just need to prove that under object identity, $\gamma_\preceq$ is also proper. 
Object identity in subsumption relative to the partial order $\prec$ implies that $|V_u| = |\{m(v)|v \in V_u\}|$ (since $v_1 \neq v_2 \implies m(v_1) \neq m(v_2)$). Thus, this implies that for two graphs $g_u = \langle V_u, E_u, l_u \rangle , g_d = \langle V_d, E_d, l_d \rangle \in G$ to subsume each other (i.e., for being equivalents), we must have that $|V_u| = |V_d|$, which then implies that $|E_u| = |E_d|$. 
Since UR2PO and UR3PO all decrease either the number of vertices or the number of edges of a graph, under object identity, a graph $g$ can never be equivalent to any of its refinements generated by UR2PO and UR3PO. 

Let us now consider UR0PO. UR0PO takes a graph $g_u$ and generates a graph $g_d$ that is identical except for the label of a vertex $v$, which is changed from a label $a$ to a label $b$, such that $b$ is more general than $a$: $b \preceq a$ and $a \neq b$. We can split the set of vertices $V_u$ into two subsets: $V_u^a = \{v \in V_u | a \preceq l_u(v) \}$ (those vertices with a label equal or more specific than $a$) and $V_u^{\overline{a}} = V_u \setminus V_u^a$. We can do the same with the vertices of $V_d$: $V_d^a = \{v \in V_d | a \preceq l_d(v) \}$ and $V_d^{\overline{a}} = V_d \setminus V_d^a$. By construction, we also know that $V_u^a = V_d^a \cup \{v\}$ (since the only change to the graph is the label of vertex $v$), and thus $|V_u^{a}| = |V_d^a| + 1$. Now, if $g_u$ is to subsume $g_d$, a mapping $m^{-1}$ from $V_u$ to $V_d$ must exist that satisfies the subsumption conditions. Given Definition \ref{def:subsumption-po}, vertices in $V_u^{a}$ must be mapped via $m^{-1}$ to different vertices (due to object identity) in $V_d^{a}$, however, since $|V_d^{a}| < |V_u^a|$, this mapping cannot exist, and thus $g_u$ cannot subsume $g_d$ if $g_d$ is generated with UR0PO.

Finally, let us consider UR1PO. Notice that object identity also implies that when $g_u$ and $g_d$ subsume each other, we have that $|E_u| = |E_d|$. This is because, the subsumption mapping $m$ between edges also defines a mapping between edges, where $m((v,w)) = (m(v), m(w))$, given that $v \neq w$ implies $m(v) \neq m(w)$, we also know that if $e_1 \neq e_2$ (where $e_1, e_2 \in E_u$), then $m(e_1) \neq m(e_2)$. Thanks to this fact, R5PO cannot generate a refinement that subsumes $g_u$. The proof is analogous to the one for UR0PO in the previous paragraph, but this time splitting the set of edges $E_u$ into two subsets, instead of the set of vertices.

Since none of the refinements generated by UR0PO, UR1PO, UR2PO, and UR3PO can be equivalent to $g_u$, we have that $\rho_\preceq$ is proper.
\end{proof}

%------------------------------------------------------------------------

\subsection{Proofs for Upward Refinement of ODLG using Trans-Subsumption ($\gamma_{t\preceq}$)}

\noindent {\bf Proposition \ref{prop:gamma-tprec-finite-complete}}. {\em
The upward refinement operator $\gamma_{t\preceq}$ defined by the rewrite rules UR0PO, UR1PO, UR2PO, UR3PO, and UR4PO above is locally finite, and complete for the quasi-ordered set $\langle G, \overrightarrow{\sqsubseteq}_{\prec} \rangle$ (where $\overrightarrow{\sqsubseteq}_{\prec}$ represents trans-subsumption relative to the partial order $\prec$).
}

\begin{proof}
Let us proof each of the properties separately:
\begin{itemize}
\item {\em $\gamma_{t\preceq}$ is locally finite}: the number of refinements generated by each rewrite rule corresponds to the number of possible values that the variables in the {\em applicability conditions} (left-hand side of the rule) can take. Thus, let us consider the different rewrite rules:
	\begin{itemize}
	\item by Proposition \ref{prop:gamma-prec-finite-complete}, UR0PO, UR1PO, UR2PO, and UR3PO produce only a finite number of refinements.
	\item UR4PO: the only two free variables are $e_1$ and $e_2$, can take at most $|E|$ possible values each. Thus, UR4PO can at most generate $|E|^2$ refinements, which is a finite number.
	\end{itemize}
\item {\em $\gamma_{t\preceq}$ is complete}: consider any two DLGs $g_u = \langle V_u, E_u, l_u \rangle , g_d = \langle V_d, E_d, l_d \rangle \in G$, such that $g_u \overrightarrow{\sqsubseteq}_{\prec} g_d$. We need to proof that we can get to $g_d$ from $g_u$ by repeated application of the refinement operator. Since $g_u \overrightarrow{\sqsubseteq}_{\prec} g_d$, we know there are two mappings $m$, $m_e$ that satisfy Definition \ref{def:trans-subsumption-po}. We will distinguish two cases:
	\begin{itemize}
	\item If $g_d \overrightarrow{\sqsubseteq}_{\prec} g_u$: then there is no need to apply the refinement operator, and we are done.
	\item If $g_d \overrightarrow{\not\sqsubseteq}_{\prec} g_u$: in this case, analogously to what we did to proof completeness in previous propositions, we will define a procedure that ensures getting to $g_u$ by repeated application of $\gamma_{t\preceq}$ starting from $g_d$ in a finite number of steps. At each step $t$ of the procedure, we will construct a new graph $g_t = \langle V_t, E_t, l_t \rangle$ using upward refinement of $g_{t-1}$, and getting one step closer to $g_u$. In the first step, $t=0$, $g_0 = g_d$:
		\begin{itemize}
			\item if $\exists w \in V_t$ such that $\nexists v \in V_u : m(v) = w$, then UR0PO can be used to generate a graph with the label of $w$ generalized one step closer to $\top$. Since no vertex of $g_u$ is mapped to $w$, changing its label does not affect any of the subsumption conditions, and the resulting graph will still be subsumed by $g_u$.
			\item Otherwise, if $\exists v \in V_u$ such that $l_u(v) \neq l_t(m(v))$, then since $g_u \sqsubseteq_\preceq g_t$, we know that $l_u(v) \preceq l_t(m(v))$, and UR0PO can be used to generalize $l_t(m(v))$ on step closer to $l_u(v)$, by selecting $b$ such that $l_u(v) \preceq b$ in order to ensure that the subsumption relation between $g_u$ and the new refinement is still satisfied (notice that such $b$ must exist since $l_u(v) \preceq l_t(m(v))$).
			\item Otherwise, if $\exists e \in \Delta_e$ such that $l_t(e) \neq \top$, then UR1PO can be used to generate a graph with the label of $e$ generalized one step closer to $\top$. Since $e$ does not belong to the cover, changing its label does not affect any of the subsumption conditions, and the resulting graph would still be subsumed by $g_u$.
			\item Otherwise, if $\exists e_2 = (w_1, w_2) \in \mathcal{C}_e$ such that $\exists e_1 \in E_u: w_1, w_2 \in m_e(e_1)$, and $l_u(e_1) \neq l_t(e_2)$, then since $g_u \sqsubseteq_\preceq g_t$, we know that $l_u(e_1) \preceq l_t(e_1)$, and UR1PO can be used to generalize $l_t(e_2)$ on step closer to $l_u(e_1)$, by selecting $b$ such that $l_u(e_1) \preceq b$ in order to ensure that the subsumption relation between $g_u$ and the new refinement is still satisfied (notice that such $b$ must exist since $l_u(e_1) \preceq l_t(e_2)$).
    		\item At this point, the steps above ensure that: 
			\begin{itemize}
				\item $\forall_{v \in V_u} l_u(v) = l_t(m(v))$, 
				\item $\forall_{w \in \Delta_v}$ $l_t(w) = \top$, 
				\item $\forall_{e \in \Delta_e} l_t(e) = \top$, and
				\item $\forall e_2 = (w_1, w_2) \in \mathcal{C}_e$ such that $\exists e_1 \in E_u: w_1, w_2 \in m_e(e_1)$, we have that $l_u(e_1) = l_t(e_2)$.
			\end{itemize}
				Now, if $\Delta_v \neq \emptyset$, then that means that there are vertices in $V_t$ to which no vertex or edge in $g_u$ is mapped. By Proposition \ref{prop:delta-leaf-or-nonbridge}, one of these two situations must arise:
			\begin{itemize}
			\item If $\exists w \in \Delta_v$ such that $w$ is a {\em leaf} (i.e., it is only connected to the rest of $g_t$ via a single edge, then, we can use UR3PO to generate a new graph by removing $w$.
			\item Otherwise, there must be a non-bridge edge $e = (w_1, w_2) \in E_t$, such that either $w_1 \in \Delta_v$ or $w_2 \in \Delta_v$. UR2PO can be used to generate a new graph by removing $e$.
			\end{itemize}
		\item Otherwise, we know that $\Delta_v = \emptyset$. Now, if $\Delta_e \neq \emptyset$, then we know that none of the edges in $\Delta_e$ can be a bridge (since $g_u$ is a connected graph), and thus, we can use UR0 to generate refinements that remove those edges one at a time.
		\item Otherwise, we know that both $\Delta_v = \emptyset$ and $\Delta_e = \emptyset$. Now only two situations might arise:
			\begin{itemize}
			\item If there is any edge $e \in E_u$ such that $m_e(e) = [w_1, ..., w_k]$, where $k>2$, then we know that no vertices in $g_u$ are mapped to any of the vertices $w_2, ..., w_{k-1}$. Since $\Delta_e = \emptyset$, the only edges that $w_i$ ($2 \leq i \leq k-1$) can have thus are those linking with $w_{i-1}$ and $w_{i+1}$. Under this circumstances, we can use UR3PO to generate a refinement removing a vertex $w_i$, where $2 \leq i \leq k-1$ (at this point, we know that vertex $w_i$ must have label $\top$).
			\item Otherwise, we can show that $g_t \overrightarrow{\sqsubseteq}_{\prec} g_u$ via a mapping $m^{-1}$ constructed in the following way: $\forall_{v_2 \in V_t} m^{-1}(v_2) = v_1 : m(v_1) = v_2$. In other words, we just have to invert the mapping. $m_e$ can also be inverted directly, since at this point in the procedure each edge $e = (v_1,v_2) \in V_u$ is mapped to a sequence of just two vertices $m_e(e) = [w_1,w_2]$, and thus, we can invert the mapping as $m_e^{-1}((w_1,w_2)) = [v_1,v_2]$. Moreover, notice that if we are not enforcing object identity, more than one vertex in $V_u$ might map to the same vertex $v_2 \in V_t$, when constructing $m^{-1}$ we just need to pick any of those vertices in $V_u$ as the mapping of $v_2$. The same thing might happen with $m_e$. It is trivial to see that this mapping satisfies the subsumption conditions in Definition \ref{def:trans-subsumption-po}.
			\end{itemize}
		\end{itemize} 
	Notice that the previous procedure always terminates in a finite number of steps because at each step we are always decreasing either the number of edges or the number of vertices of the graph (with UR2PO, UR3PO or UR4PO), and never increasing any of them, or making a label more general (with UR0PO or UR1PO) and never making it more specific. Moreover, since the number of edges and vertices are trivially lower-bounded by 0, and the number of times we can make a label more general is bounded by the size of $|L|$, the previous process must terminate in a finite number of steps.
	\end{itemize}
\end{itemize}
\end{proof}

%------------------------------------------------------------------------

\noindent {\bf Proposition \ref{prop:gamma-tprec-ideal}} {\em
The upward refinement operator  $\gamma_{t\preceq}$ defined by the rewrite rules UR0PO, UR1PO, UR2PO, UR3PO, and UR4PO above is ideal (locally finite, complete, and proper) for the quasi-ordered set $\langle G, \overrightarrow{\sqsubseteq}_{\prec} \rangle$ (where $\overrightarrow{\sqsubseteq}_{\prec}$ represents trans-subsumption relative to the partial order $\prec$), when we impose the Object Identity constraint.
}
\begin{proof}
By Proposition \ref{prop:gamma-tprec-finite-complete} $\gamma_{t\preceq}$ is already locally finite and complete, so, we just need to prove that under object identity, $\gamma_{t\preceq}$ is also proper. As shown in the proof of Proposition \ref{prop:gamma-prec-ideal}, for two graphs $g_u$ and $g_d$ to be equivalent under trans-subsumption, they must have the same exact number of vertices and edges. Since UR2PO, UR3PO, and UR4PO all reduce either the number of vertices or edges, upward refinements of a graph $g$ generated by UR2PO, UR3PO, and UR4PO can never be equivalent to $g$.

By Proposition \ref{prop:gamma-prec-ideal}, upward refinements of a graph $g$ generated by UR0PO, and UR1PO can never be equivalent to $g$ either, and thus we have that $\rho_t\preceq$ is proper.
\end{proof}

\end{document}